\documentclass[lettersize,journal]{IEEEtran}

\usepackage{amssymb}
\usepackage{amsmath}
\usepackage{amsthm}
\newtheorem{thm}{Theorem}

\newtheorem*{remark}{\it Remark}
\newtheorem*{definition}{Definition}

\usepackage{algorithmic}
\usepackage{algorithm}
\usepackage{array}
\usepackage[caption=false,font=normalsize,labelfont=sf,textfont=sf]{subfig}
\usepackage{textcomp}
\usepackage{stfloats}
\usepackage{url}
\usepackage{verbatim}
\usepackage{tikz}
\usepackage{graphicx}
\usepackage{float}             
\usepackage{subfig}                
\usepackage{overpic}
\usepackage{color}
\usepackage{booktabs}
\usepackage{threeparttable}
\usepackage{multirow}
\begin{document}

\title{LRTuckerRep: Low-rank Tucker Representation Model for Multi-dimensional Data Completion} %\thanks{This work was supported by SUSTech Presidential Postdoctoral Fellowship and the China Postdoctoral Science Foundation}
\author{Wenwu Gong and Lili Yang

\thanks{Wenwu Gong (ORCID: 0000-0002-8019-0582) is with the Department of Statistics and Data Science, Southern University of Science and Technology, Shenzhen, 518055, China.}

\thanks{Lili Yang is with the Department of Statistics and Data Science, Southern University of Science and Technology, Shenzhen, 518055, China.}
} 

\maketitle

%% Abstract
\begin{abstract}
Multi-dimensional data completion is a critical problem in computational sciences, particularly in domains such as computer vision and signal processing. Existing methods typically leverage either global low-rank approximations or local smoothness regularization, but each suffers from notable limitations: low-rank methods are computationally expensive and may disrupt intrinsic data structures, while smoothness-based approaches often require extensive manual parameter tuning and exhibit poor generalization. In this paper, we propose a novel \textbf{L}ow-\textbf{R}ank \textbf{Tucker} \textbf{Rep}resentation (LRTuckerRep) model that unifies global and local prior modeling within a Tucker decomposition. Specifically, LRTuckerRep encodes low rankness through a self-adaptive weighted nuclear norm on the factor matrices and a sparse Tucker core, while capturing smoothness via a parameter-free Laplacian-based regularization on the factor spaces. To efficiently solve the resulting nonconvex optimization problem, we develop two iterative algorithms with provable convergence guarantees. Extensive experiments on multi-dimensional image inpainting and traffic data imputation demonstrate that LRTuckerRep achieves superior completion accuracy and robustness under high missing rates compared to baselines.
\end{abstract}

%% Keywords

\section{Introduction}
\label{sec: intro}
\IEEEPARstart{I}n the era of big data and artificial intelligence, multi-dimensional data with complex structures is increasingly prevalent across diverse domains, including computer vision, signal processing, and scientific computing. Tensor representations depict complex structural information from multi-dimensional data, which plays an important role in image science \cite{IP2019} and signal processing \cite{SP2017}. However, multi-dimensional data collected in practical applications suffers from degradation and information loss, affecting image enhancement quality and traffic prediction accuracy. One of the most fundamental issues is to estimate missing values due to image corruption or sensor failure, commonly known as tensor completion (TC). Liu et al. \cite{TC2009} first introduced the TC problem in their 2009 conference paper. As a subdomain of inverse problems, the TC problem refers to estimating the multi-dimensional data $\mathcal{X} \in \mathbf{R}^{I_1 \times \cdots \times I_N} $ from its partial observations $\mathcal{T} \in \mathbf{R}^{I_1 \times \cdots \times I_N} $ under the projection operator $\Omega \in \mathbf{R}^{I_1 \times \cdots \times I_N}$. From a Bayesian perspective \cite{GLNP2022}, the TC problem can generally be expressed as a posterior distribution stated in \eqref{TC} 
\begin{equation}
\begin{aligned}
\hat{\mathcal{X}} =\ \arg \min_{\mathcal{X}} \ \mathcal{R}(\mathcal{X}),\quad  \text{s.t.,} \quad \mathcal{X}_{\Omega} = \mathcal{T}_{\Omega}, 
\end{aligned}\label{TC}
\end{equation}
where the prior $\mathcal{R}(\mathcal{X}) $ represents the intrinsic property of data, and the constraint equation enforces that $\mathcal{X}$ is consistent with $\mathcal{T}$ under the observed index $\Omega$. Real-world multi-dimensional data are inherently rich in structural redundancy, often characterized by strong global correlations and local similarities \cite{LS2023}. These two complementary properties not only reflect the intrinsic patterns underlying the data but also offer crucial inductive biases that can be exploited to improve the performance of TC. As such, effectively modeling global and local structures has become a central motivation in developing TC methods. 

The low rankness reveals the global correlations among tensors, which is the main focus of the prior modeling for TC problems. Liu et al. \cite{HaLRTC2013} introduced a method that uses the summation of nuclear norms of tensor unfolding matrices for the TC problem. Although this approach offers robust theoretical guarantees for completion, it is computationally expensive. In contrast, parallel methods \cite{Tmac2015, TT2017} used fixed low-rank matrix factorization and inexact optimization techniques to reduce computational complexity. However, these unfolding matrix factorization methods may disrupt the tensor structure and perform poorly when faced with highly corrupted tensors.

An alternative approach is the low-rank tensor decomposition method, which preserves the tensor structure and avoids high computational costs. Zhao et al. \cite{BayesianCP2022} developed a sparsity-induced low-rank CANDECOMP/PARAFAC (CP) decomposition that automatically determines the tensor rank. Li et al. \cite{tSVDTR2022} introduced a novel low-rank Tubal factorization that utilizes a nonlinear transformation to improve the tensor approximation. Ji et al. \cite{LogTucker2017} developed a nonconvex Tucker rank approximation for the low-rank tensor completion problem. Xie et al. \cite{KBR2018} proposed a Tucker-based nonconvex relaxation method known as Kronecker-Basis-Representation to improve TC precision. Unlike the existing Tucker-based methods, which primarily focus on achieving low-rank representations through unfolding matrices, our method introduces a novel low-rank Tucker representation through Tucker core sparsity combined with weighted factor matrix nuclear norm. Moreover, previous approaches often overlook the importance of preserving local similarities within the data, thereby limiting their performance in real-world tensor data completion.

Local similarities are typically characterized by smoothness, such as tensor gradient \cite{LATC2021, LS2023} and factor gradient \cite{Auxiliary2012, ARTD2023}. Recent research has made significant progress in joint global and local priors modeling, including nuclear norm-based \cite{SNNTV2017}, matrix factorization-based \cite{STMac2016}, and tensor decomposition-based \cite{SPC2016, TTTV2019, TubalLSTC2021}. Additionally, several smooth Tucker models have investigated the joint priors of low rankness and smoothness \cite{ESP2020, SBCD2022, ARTD2023}. The Laplacian-based factor gradient regularization proposed by \cite{ESP2020} represents a significant advancement in capturing local structure within Tucker factor spaces. However, this approach relies on manually tuned regularization parameters that are sensitive to specific datasets and application scenarios, limiting its generalizability. Moreover, most existing models struggle to effectively balance low-rank representation with the preservation of intrinsic smoothness in Tucker components, thus failing to fully leverage the benefits of joint prior optimization. For an illustrative example, Fig.~\ref{fig0} shows the inpainting of the RGB image. 

\begin{figure}[htbp]
  \centering
  \includegraphics[width=1\linewidth]{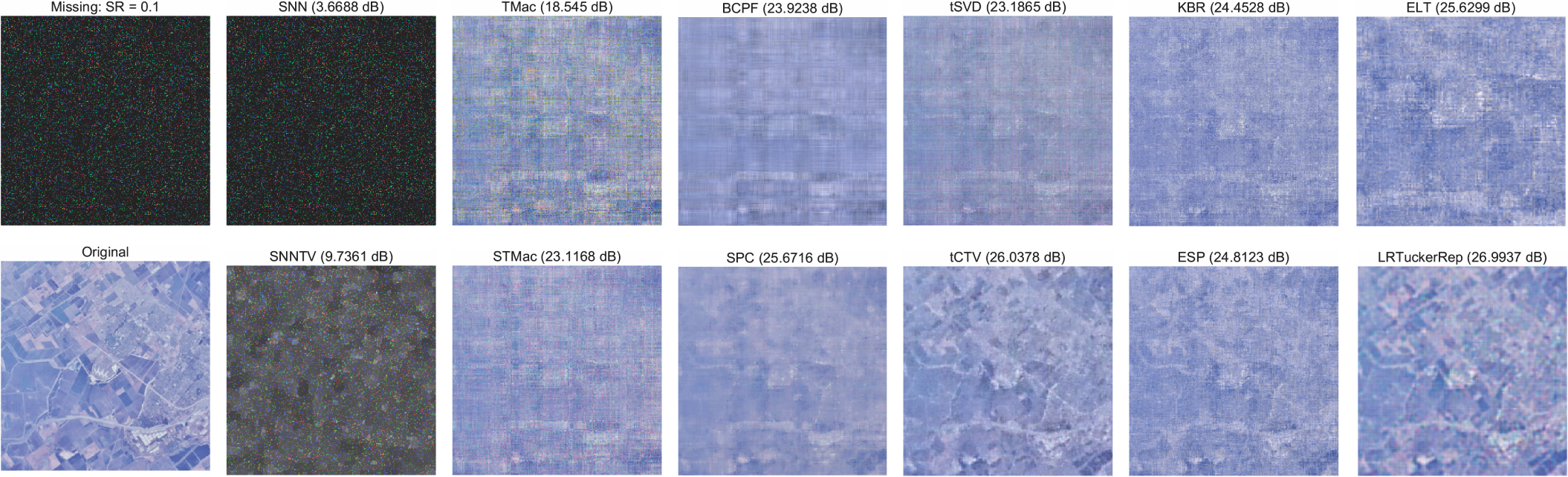}
  \caption{Model performance of joint low rankness and smoothness in the inpainting task with sample ratio (SR) is 10\%. The recovery results under low rankness prior: SNN \cite{HaLRTC2013}, TMac \cite{Tmac2015}, BCPF \cite{BGCP2015}, tSVD \cite{tSVD2017}, KBR \cite{KBR2018}, our Low-rank Tucker representation; Recovery obtained by several methods built under joint low rankness and smoothness priors: SNNTV \cite{SNNTV2017}, STMac \cite{STMac2016}, SPC \cite{SPC2016}, tCTV \cite{LS2023}, ESP \cite{ESP2020}, and our proposal. It can be seen that smoothness enhances recovery performance, and our proposal performs better.}
  \label{fig0}
\end{figure}

In addition to developing tensor optimization models, it is crucial to design efficient solving algorithms. Tucker-based TC models are generally nonconvex, and the prior structures chosen are often nonsmooth. This complexity makes it challenging to analyze the convergence of multiblock nonconvex optimization problems, and existing algorithms typically result in high computational complexity. Notable solving algorithms include proximal alternative minimization (PAM) \cite{PAM2015} and the proximal alternative direction method of multipliers (PADMM) \cite{ProADMM2015} have shown effectiveness and efficiency in nonconvex optimization, particularly when approximation strategies are used \cite{BCDXu2013, PALM2014}. However, a lack of systematic analysis of tensor completion algorithms that leverage Tucker decomposition while ensuring convergence remains.

This work aims to tackle three fundamental challenges in the context of Tucker-based tensor modeling for multi-dimensional data completion: 1) Existing methods primarily achieve low-rank representations through matrix unfolding operations. It is crucial to develop a novel Tucker-based low-rank representation framework that effectively exploits the high-order dependencies while maintaining the native structural integrity of the data. 2) To capture local smoothness within Tucker components remains nontrivial, as current methods frequently rely on manually tuned regularization parameters that are sensitive to datasets and task configurations. Therefore, a critical objective is to design an adaptive smoothness characterization mechanism that requires no hyperparameter tuning, thereby ensuring more stable and generalizable data completion performance. 3) Tucker-based models typically lead to multi-block nonconvex optimization problems, further complicated by the incorporation of structural priors. Consequently, there is a critical need to develop efficient and provably convergent optimization algorithms that can handle the model’s nonconvexity while preserving computational scalability and accuracy.

This paper addresses the joint prior modeling of low rankness and smoothness in Tucker-based tensor completion problems. The main contributions are summarized as follows.
\begin{enumerate}
    \item \textit{Low-rank Tucker measure:} Inspired by \cite{KBR2018}, we use the weighted factor matrix nuclear norm and the Tucker core $l_1$ norm to characterize the low Tucker rank. The low-rank Tucker measure solves the imbalance of unfolding tensor matrices and offers a novel interpretation of low rankness based on tensor sparsity. Furthermore, the nuclear norm weights are self-adaptive, and a trade-off parameter is established to balance the roles of low rankness and sparsity.
    \item \textit{Smoothness:} We improve tensor completion performance by capturing a smooth structure through the factor gradient, which employs a parameter-tuning-free Laplacian regularization on the factor matrix to characterize smoothness.
    \item \textit{Efficient optimization:} The Proximal Alternating Linearized Minimization (PALM) and Proximal Alternating Direction Multiplier (ProADM) algorithms are proposed for solving the proposed model. Numerically, these algorithms are designed to be single-looped, which makes programming very easy. Theoretically, both the PALM and ProADM algorithms are shown to converge to a critical point globally.
    \item \textit{Numerical results:} Extensive experiments evaluate the performance of our proposal in image data inpainting and traffic data imputation. The numerical results demonstrate that the proposed model exhibits strong generalization capabilities and outperforms existing baselines.
\end{enumerate}

We first briefly review some works on the Tucker-based tensor completion models in Sect.~\ref{sec: RW}. In Sect.~\ref{sec: main}, we present our main results, including the proposed low-rank Tucker representation model, two solving algorithms, and convergence results. In Sect.~\ref{sec: experiments}, we provide numerical results of our proposal and compare them with baseline methods for image data inpainting and traffic data imputation. Sect.~\ref{sec: conclusions} presents conclusions and future work.

\section{Related works}
\label{sec: RW}
Numerous studies have investigated the Tucker-based approaches \cite{SBCD2022, ARTD2023} for the tensor completion (TC) problem. In this section, we review several related Tucker-based TC models that incorporate joint low rankness and smoothness, focusing on applications in multi-dimensional image inpainting and traffic imputation. The key attributes of various Tucker-based TC methods are summarized in Table~\ref{tab1}.

\begin{table*}[htbp]
  \footnotesize
  \centering
  \caption{Some existing Tucker-based TC methods utilizing different priors} \label{tab1}
  \label{RW}
  \begin{tabular}{@{}lccccccc@{}} 
    \toprule
    \multirow{2}{*}{\textbf{TC methods}}
    & \multicolumn{2}{c}{Low rankness} & \multicolumn{2}{c}{Smoothness} \\ 
    & nuclear norm & sparsity & tensor gradient & factor gradient\\
    \midrule
    Ours & \checkmark & \checkmark & & \checkmark \\
    LRSETD \cite{LRSETD2024} & \checkmark & \checkmark & \checkmark & \\
    SparsityTD \cite{ARTD2023} &  & \checkmark & & \checkmark\\
    LSMTLT \cite{BayesianTucker2022} &  & \checkmark & \checkmark & \\
    SBCD \cite{SBCD2022} & \checkmark & \checkmark & & \checkmark \\
    ESP \cite{ESP2020} & \checkmark & &  & \checkmark \\
    gHOI \cite{gHOI2016} & \checkmark & & \checkmark & \\
    STDC \cite{STDC2014} & \checkmark &  &  & \checkmark \\
    \bottomrule
  \end{tabular}
  \begin{threeparttable}
    \checkmark denotes the mentioned priors that have been considered
  \end{threeparttable}
\end{table*}

Chen et al. \cite{STDC2014} pioneered a smooth graph-regularized structure combined with low-rank Tucker decomposition for image data recovery. Although the TC accuracy has improved significantly compared to earlier methods, it often necessitates extensive hyperparameter tuning. Liu et al. \cite{gHOI2016} applied low rankness to unfolding the core tensor while ensuring smoothness through the constraint of orthogonal factor matrices, reducing computational complexity. Li et al. \cite{SNNTV2017} integrated the spatial tensor gradient into the Tucker model to exploit the piecewise smooth structure along the spatial dimensions of the visual data, demonstrating the effectiveness of inpainting tasks. Xue et al. \cite{ESP2020} utilized a tensor sparsity measure to encode the low-rank property and applied the $l_2$ norm factor gradient to capture local properties, enhancing tensor completion performance. Pan et al. \cite{LRSETD2024} proposed combining a Tucker sparsity term with a tensor gradient to improve both the inpainting and the traffic imputation performance. 

Most existing Tucker-based methods require the tensor rank to be predefined and often lack a clear interpretation of the low rankness. The recent low Tucker rank model \cite{SBCD2022} introduced a novel approach that does not require predefined ranks, while Xue et al. \cite{BayesianTucker2022} proposed a Bayesian three-layer transform structure to measure Tucker sparsity and improve TC performance. Furthermore, Gong et al. \cite{ARTD2023} proposed a sparsity-based Tucker decomposition model for the restoration of color images and the imputation of spatiotemporal traffic data, focusing on tensor sparsity under the full Tucker rank. However, coupling low rankness and smoothness priors leads to redundant information, complicating hyperparameter tuning. A novel low-rank Tucker representation model that integrates smoothness without requiring parameter tuning remains an open study area.

\section{Tucker-based prior modeling}
\label{sec: main}

\subsection{Notations and preliminaries}
Given a tensor ${\mathcal{X} \in \mathbb{R}^{I_1 \times I_2 \times \cdots \times I_N}}$, it can be decomposed into a core tensor $\mathcal{G}\in \mathbb{R}^{I_1 \times I_2 \times \cdots \times I_N}$ multiplying a matrix $\mathbf{U}_{n} \in \mathbb{R}^{I_n \times I_n}$ along each mode, that is, $\mathcal{X} = \mathcal{G} \times_{n=1}^{N} {\mathbf{U}_{n}}$. In addition, the matrix formulation is shown as $\mathbf{X}_{(n)} =  \mathbf{U}_{n} \mathbf{G}_{(n)} \mathbf{V}_{n}^{\mathrm{T}}$, where $\mathbf{V}_{n} = \left(\mathbf{U}_{N} \otimes \cdots \otimes \mathbf{U}_{n+1} \otimes \mathbf{U}_{n-1} \otimes \cdots \otimes \mathbf{U}_{1}\right)$ and the superscript `$\mathrm{T}$' represent the matrix transpose. It is easy to verify that $\operatorname{vec}(\mathcal{X}) = \left(\mathbf{U}_{N} \otimes \cdots \otimes \mathbf{U}_{n} \otimes \cdots \otimes \mathbf{U}_{1}\right) \operatorname{vec}({\mathcal{G})} = \otimes_{n=N}^1 \mathbf{U}_{n} \operatorname{vec}({\mathcal{G})}$, and 
Table~\ref{tab2} presents all notations used in this paper.

\begin{table}[ht]
    \caption{Notations} \label{tab2}
    \centering
    ${
        \begin{array}{r|l}
            \hline \hline 
            \mathcal{X}, \mathbf{U}, \alpha  &  \text{A tensor, matrix and real value, respectively.} \\
            \Omega, \bar{\Omega} & \text {Observed index set and its complement.} \\
            {\mathcal{S}_{\eta}(x)} & \text{Shrinkage operator with} \ \eta \text{ in component-wise}. \\ 
            {\mathcal{D}_{\eta}(\mathbf{U})} & \text{Singular value decomposition (SVD) shrinkage of matrix } \mathbf{U}. \\
            \mathcal{X}_{\Omega} & \text{Observed entries supported on the observed index}. \\
            \times_n & \text {Mode-n product.} \\
            \otimes & \text {Kronecker product.} \\
            \operatorname{tr} & \text {Trace operator.} \\
            \sigma_{j}(\mathbf{U}) & \text {the $j$th largest singular value of} \ \mathbf{U}. \\
            \left\| \cdot \right\|_F^2 & \text {Frobenius norm.} \\
            \left\| \cdot \right\|_{\ast} & \text {Nuclear norm.} \\
            \left\| \cdot \right\|_{1} & L_1 \text{norm of tensor vectorization.} \\
            \left\| \cdot \right\|_2 & \text {Spectral norm.} \\
            \mathbf{X}_{(n)}  & \text {Mode-n unfolding of tensor} \ \mathcal{X}. \\
            \hline \hline
        \end{array}
    }$
\end{table}

\begin{definition}
     Let ${f: \mathbb{R}^{d} \rightarrow(-\infty, \infty]}$ be a proper and lower semicontinuous function. Given ${\tilde{x} \in \mathbb{R}^{d}}$, ${L>0}$, the proximal operator associated with ${f}$ is defined by:
        \begin{equation}
        \operatorname{prox}_{L}^{f}(\tilde{x}):=\operatorname{argmin} f(x)+\frac{L}{2}\|x-\tilde{x}\|^{2}, \  x \in \mathbb{R}^{d}. \label{proximal}
        \end{equation}
\end{definition}

\begin{remark} 
    The mapping ${\operatorname{prox}_{L}^{f}(\cdot)}$ only depends on ${f}$ and has a closed form in many applications. Let $\mathcal{D}_{\tau}(\cdot)$ and $\mathcal{S}_{\tau}(\cdot)$ represent the singular value shrinkage \cite{SVT2013} and the soft-thresholding operator \cite{SNTD2015}, respectively. We have the nuclear norm minimization
    \begin{equation}
        \operatorname{prox}_{1/\tau}^{\|\cdot\|_*}(\tilde{\mathbf{U}}) = \underset{\mathbf{U}}{\operatorname{argmin}} \ \frac{1}{2}\|\mathbf{U}-\tilde{\mathbf{U}}\|_{F}^{2}+\tau\|\mathbf{U}\|_{*} = \mathcal{D}_\tau(\tilde{\mathbf{U}}), \label{DO}
    \end{equation}
    and LASSO problem
    \begin{equation}
        \operatorname{prox}_{1/\tau}^{\|\cdot\|_1}(\tilde{\mathcal{G}}) = \underset{\mathcal{G}}{\operatorname{argmin}} \ \frac{1}{2}\|\mathcal{G}-\tilde{\mathcal{G}}\|_{F}^{2}+\tau\|\mathcal{G}\|_{1} = \mathcal{S}_\tau(\tilde{\mathcal{G}}). \label{SO}
    \end{equation}
\end{remark}

\begin{definition}
    (i) A set \( \mathcal{D} \subset \mathbb{R}^{n} \) is called semianalytic \cite{KL2013} if it can be represented as
    \begin{equation*}
    \mathcal{D} = \bigcup_{i=1}^{s} \bigcap_{j=1}^{t}\left\{\mathbf{x} \in \mathbb{R}^{n}: p_{i j}(\mathbf{x})=0, q_{i j}(\mathbf{x})>0\right\},
    \end{equation*}
    where \( p_{i j}, q_{i j} \) are real real-analytic functions for \( 1 \leq i \leq s, 1 \leq j \leq t \). (ii) The set \( A \) is called subanalytic if each point in \( \mathbb{R}^{n} \) admits a neighborhood \( B \) and
    \begin{equation*}
    A \cap B =\left\{x \in \mathbb{R}^{n}:(x, y) \in \mathcal{D} \right\},
    \end{equation*}
    where \( \mathcal{D} \) is a semianalytic subset bounded by \( \mathbb{R}^{n} \times \mathbb{R}^{m} \) for some \( m \geq 1 \). (iii) A function \( \Phi \) is called subanalytic if its graph \( \operatorname{Gr}(\Phi) = \{(\mathbf{x}, \Phi(\mathbf{x})): \mathbf{x} \in \operatorname{dom}(\Phi)\} \) is a subanalytic set.
\end{definition}

\begin{definition}
    Let ${\Phi: \mathbb{R}^{d} \rightarrow(-\infty, +\infty]}$ be proper and lower semicontinuous, ${\varphi \in C^{1}:[0, \eta) \rightarrow \mathbb{R}_{+} }$, ${\eta \in(0,+\infty]}$ be a concave function.

    \noindent (i) The function ${\Phi}$ has the Kurdyka-Łojasiewicz (KL) property \cite{KL2013} at $\hat{x}$ if there exists a neighborhood ${\epsilon_x}$ of ${\hat{x} \in \left\{x \in \mathbb{R}^{d}: \partial \Phi(x) \neq \emptyset\right\}}$, such that $\varphi{\prime}(\Phi(x)-\Phi(\hat{x})) \operatorname{dist}(0, \partial \Phi(x)) \geq 1$ for all $x \in \{\epsilon_x \cap[\Phi(\hat{x})<\Phi(x)<\Phi(\hat{x})+\eta]\}$ holds;

    \noindent (ii) If ${\Phi}$ satisfies the KL property at each point of ${\partial \Phi}$ then ${\Phi}$ is called a KL function.
\end{definition}

\begin{remark} 
    If ${\Phi}$ is subanalytic or locally strongly convex, it satisfies the KL property at any point of ${\partial \Phi}$, named the KL function \cite{KL2013}.
\end{remark}

\subsection{Low-rank Tucker representation} 

\subsubsection{Motivations} The motivations of the model \eqref{Model} are summarized below. As illustrated in Fig.~\ref{framework} (a-2, b-2, c-2), many elements of the Tucker core are zero, and factor matrices exhibit low-rank structures. On the one hand, the $l_1$ norm can effectively represent the sparsity \cite{ARTD2023}. On the other hand, minimizing the nuclear norm of the factor matrix in Tucker decomposition is equivalent to optimizing the nuclear norm of the unfolding matrix \cite{SBCD2022}. Therefore, incorporating a sparse Tucker core and low-rank factor matrices is advantageous in depicting the low Tucker rank. Furthermore, Fig.~\ref{framework} (a-3, b-3, c-3) displays the statistical distributions of the mode-n matrices. The histogram demonstrates that the Laplacian measure applied to the factor matrices can effectively encode the smoothness. Analytically, our proposal expands on \cite{KBR2018, ARTD2023} using Laplacian-based factor gradients. Conceptually, our approach closely relates to \cite{BayesianTucker2022} but understands the model well using a novel low-rank structure. Table~\ref{RW} provides a detailed description of the relevant Tucker-based completion methods. 

\begin{figure}[htbp]
  \centering
  \includegraphics[width=1\linewidth]{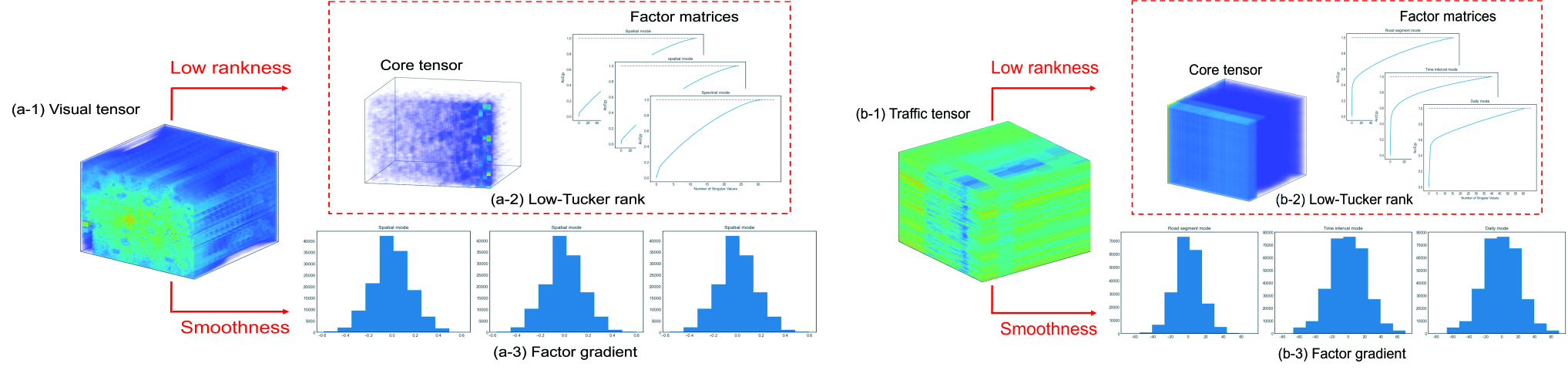}
  \caption{Visualization of low rankness and smoothness priors in the low-rank Tucker representation model.}
  \label{framework}
\end{figure}

\subsubsection{Low-Tucker rank relaxation} Considering that Tucker is an effective and efficient model to express low rankness, we use the factor matrix nuclear norm and the Tucker core $l_1$ norm relaxation, named \textbf{E}nhanced \textbf{L}ow-rank \textbf{T}ucker decomposition (ELT), to characterize low rankness.
\begin{equation}
\begin{aligned}
 \underset{\mathcal{G},\{\mathbf{U}_{n}\}, \mathcal{X}}{\operatorname{min}} & \ (1-\alpha) \prod_{n=1}^{N} \left\|\mathbf{U}_{n}\right\|_{*}+ \alpha\|\mathcal{G}\|_{1}, \\ \quad & \text { s.t.}, \ \mathcal{X} = \mathcal{G} \times_{n=1}^{N} \mathbf{U}_{n}, \ \ 0< \alpha < 1.
\end{aligned} \label{LRSTD}
\end{equation}
The nuclear norm of the factor matrix $\|\mathbf{U}_{n}\|_{*} = \sum_{j} \sigma_{j}(\mathbf{U}_{n})$ and the $l_1$ norm of Tucker core is denoted as $\|\mathcal{G}\|_{1}=\sum_{i_{1}, \ldots, i_{N}}\left|\mathcal{G}_{i_{1} \cdots i_{N}}\right|$.
\begin{remark} 
    The product function of \eqref{LRSTD} is nonconvex, which represents the block size of the Tucker core \cite{KBR2018} and is hard to solve; thus, we use the weighted factor matrix nuclear norm summation term in \eqref{Model}. 
\end{remark}

\subsubsection{Factor gradient} In recent work \cite{ESP2020, ARTD2023}, the factor gradient characterizes the local properties of the tensor. The method considers auxiliary matrices in factor matrices to find the low-dimensional representation of the tensors. To encode tensor smoothness, we propose using the Laplacian-based factor gradient. For example, we start with the unfolding tensor matrix $\mathbf{X}_{n} \in \mathbb{R}^{I_n\times \prod_{j \neq n} I_j}$ and construct the similarity matrix ${\mathbf{W}_{n} \in \mathbb{R}^{I_n\times I_n}}$ using the kernel weight $w_{i j} = e^{-\left(\left\|{x}_{i}-{x}_{j}\right\|^{2}\right)}$ for each row $x_{i} \in \mathbf{X}_{n}$. Then, the factor gradient \eqref{S} is constructed to capture an optimal low-dimensional representation $\mathbf{U}_{n}$ for $\mathbf{X}_{n}$.

\begin{equation}
   {\sum_{i=1}^{I_n} \sum_{j=1}^{I_n} w_{i j}\left\|\mathbf{u}_{i}-\mathbf{u}_{j}\right\|_{2}^{2} =  \operatorname{tr}\left(\mathbf{U}_{n}^{\mathrm{T}} {\mathbf{L}_{n}} \mathbf{U}_{n}\right)}, \ \mathbf{L}_{n} = \mathbf{D}_{n}-\mathbf{W}_{n}, \label{S}
\end{equation}
where $\mathbf{u}_{i}$ is the column vector of $\mathbf{U}_{n}^{\mathrm{T}}$ and ${\mathbf{D}_{n} \in \mathbb{R}^{I_n \times I_n}}$ is a diagonal matrix with diagonal elements ${d_{i i}=}$ ${\sum_{j=1}^{I_n} w_{i j},}$ $i=1, \ldots, I_n$. Note that $\mathbf{L}_{n}$ is a Laplacian matrix designed based on some prior knowledge, which enforces the smoothness of the low-dimensional feature $\mathbf{U}_{n}$ and captures the local properties in tensors \cite{ARTD2023}.

\begin{remark}
    The Laplacian-based matrix is constructed using data information, which reveals local properties and captures latent correlations. Consequently, the proposed factor gradient effectively represents the smooth structure of the tensors in each mode.
\end{remark}

\subsubsection{The proposed model}
Tucker components have been verified to portray the low rankness and smoothness of tensor data \cite{KBR2018, ESP2020, SBCD2022, ARTD2023}. However, these papers do not explain the Tucker low-rank structure well and lack discussion of the smoothness parameters. In this paper, we propose a low-rank Tucker representation model that simultaneously addresses low rankness and smoothness in the tensor completion problem. The proposed model offers a novel interpretation of low rankness and achieves high-precision performance without hyperparameter tuning.

Suppose $ \mathcal{G} $ is an $N$-th order tensor of the same size as $ \mathcal{X} \in \mathbf{R}^{I_1 \times \cdots \times I_N}$ and each $\mathbf{U}_{n} $ denotes an $ I_{n} \times I_{n}$ matrix. Inspired by the low-Tucker rank and factor gradient, we propose a novel \textbf{L}ow-\textbf{R}ank \textbf{Tucker} \textbf{Rep}resentation (LRTuckerRep) model \eqref{Model} to address the complementary role of the low rankness and smoothness priors in the tensor completion problem. Mathematically, the proposed LRTuckerRep model for the TC problem is formulated as 
\begin{equation}
\begin{aligned}
\underset{\mathcal{G},\{\mathbf{U}_{n}\}, \mathcal{X}}{\operatorname{min}} \ &  \left(1 - \alpha\right) \sum_{n=1}^{N} \omega_n \left\|\mathbf{U}_{n}\right\|_{*}+ \alpha \|\mathcal{G}\|_{1}  \\ & \quad + \sum_{n \in \Gamma} \frac{\beta_{n}}{2} \operatorname{tr}\left(\mathbf{U}_{n}^{\mathrm{T}} \mathbf{L}_n \mathbf{U}_{n}\right)
\\  & \text { s.t.}, \quad \mathcal{X} = \mathcal{G} \times_{n=1}^{N} \mathbf{U}_{n}, \ \  \mathcal{X}_{\Omega} = \mathcal{T}_{\Omega},
\end{aligned} \label{Model}
\end{equation}
where $0< \alpha< 1$, $\Gamma$ is a prior set determining smoothness along tensor modes, and the hyperparameters $\{\omega_n\}, \{\beta_n\}$ are self-adaptive via \eqref{hypers}
\begin{equation}
\begin{aligned}
   & \beta_{n} = \frac{\rho_n}{\sum_{n=1}^{N} \rho_n}, \, \rho_n = \frac{\sigma_1 (\mathbf{X}_{(n)})}{2*\sigma_1 (\mathbf{L}_n)}, \\ & \omega_{n}= \prod_{i=1,i \neq n}^{N} \frac{1}{R_{i}}, \, R_i=\sum_{j} \sigma_j \left(\mathbf{U}_{i}\right). 
\end{aligned}\label{hypers}
\end{equation}
For the TC problem, we update the tensor $\mathcal{X}$ by 
\begin{equation}
    \hat{\mathcal{X}} = \mathcal{T}_{\Omega} + \left(\hat{\mathcal{G}} \times_{n=1}^{N} \hat{\mathbf{U}}_{n}\right)_{\bar{\Omega}}. \label{StepX}
\end{equation}

\begin{remark}
The parameter $\beta_n$ in the proposed LRTuckerRep model governs the trade-off between low rankness and smoothness for each tensor mode and is inherently data-dependent. Specifically, low rankness is characterized via unfolding matrices, while smoothness is captured through Laplacian matrices. To adaptively balance these two priors, we set $\beta_n$ based on the ratio of the largest eigenvalues of the corresponding unfolding and Laplacian matrices. This eigenvalue-based scaling reflects the relative significance of global and local information in each mode. Numerical results demonstrate that this parameter-setting strategy leads to improved performance and stable convergence across various multi-dimensional data completion tasks.
\end{remark}

\begin{remark}
The proposed LRTuckerRep model offers a novel perspective on low Tucker rank by integrating a self-adaptive weighted nuclear norm of factor matrices with a sparse Tucker core. This formulation enables the simultaneous encoding of both low rankness and smoothness priors, leveraging the factor matrices’ ability to capture global correlations and local structural patterns. As shown in Fig.~\ref{fig0}, LRTuckerRep achieves superior performance in tensor completion tasks. Unlike prior methods such as STDC \cite{STDC2014}, gHOI \cite{gHOI2016}, ESP \cite{ESP2020}, SBCD \cite{SBCD2022}, and LRSETD \cite{LRSETD2024}, LRTuckerRep enforces low rankness and smoothness through distinct yet complementary regularizations within the Tucker decomposition. Besides, if $\alpha=1$, the proposed model degrades into a SparsityTD \cite{ARTD2023, STRTD2023} model.
\end{remark}

\subsection{Optimization algorithms}
\subsubsection{PALM-based algorithm} To solve the nonconvex optimization problem described in \eqref{Model}, we introduce $\lambda$ (an increasing sequence of positive penalty parameters) and form a multi-convex block optimization problem \eqref{PALM}, where each subproblem can be solved using a linearized proximal gradient method \cite{PALM2014, BCDXu2013}.
\begin{equation}
\begin{aligned}
\underset{\mathcal{G},\{\mathbf{U}_{n}\}, \mathcal{X}}{\operatorname{min}} & \  (1 - \alpha) \sum_{n=1}^{N} \omega_n \left\|\mathbf{U}_{n}\right\|_{*} + \alpha \|\mathcal{G}\|_{1} \\ & +\sum_{n \in \Gamma} \frac{\beta_{n}}{2} \operatorname{tr}\left(\mathbf{U}_{n}^{\mathrm{T}} \mathbf{L}_n \mathbf{U}_{n}\right) + \frac{\lambda}{2}\left\|\mathcal{G} \times_{n=1}^{N} \mathbf{U}_{n} -\mathcal{X}\right\|_{\mathrm{F}}^{2}
\label{PALM}
\end{aligned}
\end{equation}
Let $f_{n}\left(\mathbf{U}_{n}\right) = (1 - \alpha) \ \omega_n \left\|\mathbf{U}_{n}\right\|_{*}$, $f_{N+1}\left(\mathcal{G}\right) = \alpha \|\mathcal{G}\|_{1}$, and 
$H \left(\mathcal{G}, \{\mathbf{U}_{n}\}, \mathcal{X}\right) = \sum_{n \in \Gamma} \frac{\beta_{n}}{2} \operatorname{tr}\left(\mathbf{U}_{n}^{\mathrm{T}} \mathbf{L}_n \mathbf{U}_{n}\right) + \frac{\lambda}{2}\left\|\mathcal{G} \times_{n=1}^{N} \mathbf{U}_{n} -\mathcal{X}\right\|_{\mathrm{F}}^{2}$, then we have
\begin{equation}
    \begin{aligned}
    \underset{\mathcal{G},\{\mathbf{U}_{n}\}, \mathcal{X}}{\operatorname{min}} \Phi(\mathcal{G},\{\mathbf{U}_{n}\}, \mathcal{X}) & \ = \sum_{n=1}^{N} f_{n}\left(\mathbf{U}_{n}\right) + f_{N+1}\left(\mathcal{G}\right) \\
    & \qquad + H\left(\mathcal{G}, \{\mathbf{U}_{n}\}, \mathcal{X}\right), 
    \end{aligned}\label{Alg}
\end{equation}
which consists of $N+2$ blocks and can be solved iteratively by minimizing the quadratic approximation to a block-convex and differentiable function $H$ with a nonsmooth function $f$. 

\begin{remark}
    The Tucker components' constraint ensures that \eqref{PALM} is well-defined. On the one hand, the regularization term ensures the uniqueness of the Tucker component \cite{SNTD2015}. On the other hand, $\{f_{n}\}$ are convex functions, which guarantee that the proximal mapping step admits a closed-form solution. Furthermore, the multi-convex implies that the partial gradients of $H\left(\mathcal{G}, \{\mathbf{U}_{n}\}, \mathcal{X}\right)$ are Lipschitz continuous, which guarantees that the solution set is nonempty.
\end{remark} 

When given an update order $\mathcal{G}, \{\mathbf{U}_{n}\}, \mathcal{X}$ for the problem \eqref{Alg}, the proximal operators \eqref{DO} and \eqref{SO} give the following Proposition 1 and 2. 

\noindent \textbf{Proposition 1.} 
Given any bounded matrices $\{\mathbf{U}_{n}\} $ and $\mathcal{X}$, we approximate $H\left(\mathcal{G}\right)$ around the bounded extrapolated point $\tilde{\mathcal{G}}$ and find out:
 \begin{equation}
  \begin{aligned}
     \hat{\mathcal{G}} & = \underset{\mathcal{G}}{\operatorname{argmin}} f_{N+1}\left(\mathcal{G}\right) + H\left( \mathcal{G}\right) \\ & \approx \mathcal{S}_{\frac{\alpha}{L_{\mathcal{G}}}}\left(\tilde{\mathcal{G}}-\frac{1}{L_{\mathcal{G}}} \nabla_{\mathcal{G}} H\left(\tilde{\mathcal{G}}\right)\right), 
\end{aligned}\label{G}
\end{equation}
$\nabla_{\mathcal{G}} H(\mathcal{G})$ is Lipschitz continuous, and the Lipschitz constant is bounded and given by $L_\mathcal{G} = \lambda \left\|\otimes_{n=N}^{1} \mathbf{U}_{n}^{\mathrm{T}} \mathbf{U}_{n}\right\|_{2}= \lambda\prod_{n=1}^{N}\left\|\mathbf{U}_{n}^{\mathrm{T}} \mathbf{U}_{n}\right\|_{2}$.
%\begin{equation}
    %\nabla_{\mathcal{G}} H(\mathcal{G}) = \lambda \left(\mathcal{G} \times_{1} {\mathbf{U}_{1}^{\mathrm{T}}\mathbf{U}_{1}\times_2 \cdots \times_{N} \mathbf{U}_{N}^{\mathrm{T}}\mathbf{U}_{N}} - \mathcal{T} \times_{1} {\mathbf{U}_{1}^{\mathrm{T}} \times_2 \cdots \times_{N} \mathbf{U}_{N}^{\mathrm{T}}}\right). \label{dG}
%\end{equation}

\noindent \textbf{Proposition 2.} 
For given other variables under mode-$n$ unfolding, we approximate $\mathbf{U}_n$ around the extrapolated point $\tilde{\mathbf{U}}_n$ and result in 
\begin{equation}
 \begin{aligned}
     \hat{\mathbf{U}}_n & = \underset{\mathbf{U}_{n}}{\operatorname{argmin}} f_{n}\left(\mathbf{U}_n\right) + H\left(\mathbf{U}_{n}\right) \\ & \approx \mathcal{D}_{\frac{(1-\alpha)\omega_n}{L_{\mathbf{U}_n}}}\left(\tilde{\mathbf{U}}_n-\frac{1}{L_{\mathbf{U}_n}} \nabla_{\mathbf{U}_n} H\left(\tilde{\mathbf{U}}_n\right)\right), 
\end{aligned}\label{U}
\end{equation}
$\nabla_{\mathbf{U}_{n}} H(\mathbf{U}_{n})$ is Lipschitz continuous and the $L_{\mathbf{U}_{n}}$ denotes as
\begin{equation*}
L_{\mathbf{U}_{n}} = \lambda \left\|{\mathbf{G}_{(n)}} {\mathbf{V}}^{\mathrm{T}}_{n} {\mathbf{V}}\mathbf{G}_{(n)}^{\mathrm{T}}\right\|_{2} + \beta_{n} \left\| {\mathbf{L}_{n}}\right\|_{2},
\end{equation*}
%\begin{equation}
%\nabla_{\mathbf{U}_{n}} H(\mathbf{U}_{n}) = \lambda \left({\mathbf{U}_{n}} {\mathbf{G}_{(n)}} {\mathbf{V}}^{\mathrm{T}}_{n} {\mathbf{V}}\mathbf{G}_{(n)}^{\mathrm{T}} - \mathbf{X}_{(n)} {\mathbf{V}}\mathbf{G}_{(n)}^{\mathrm{T}}\right) + \beta_{n} \mathbf{L}_n{\mathbf{U}_{n}} \label{dU},
%\end{equation}
where $\mathbf{G}_{(n)}$ refers to the matrix obtained by unfolding tensor $\mathcal{G}$ along mode-$n$.

Guided by these results in Proposition 1 and Proposition 2, we explicitly employ the block coordinate descent framework to devise our proposed algorithm, Proximal Alternating Linearized Minimization (PALM). In the PALM-based solver, we adopt an initial strategy to reduce errors by randomly generating and normalizing ${\mathbf{U}_n}$. Since model \eqref{Model} is nonconvex, choosing extrapolation points is crucial in our algorithm design \cite{APG2015}. Following insights from \cite{FISTA2022}, we set the extrapolated points by combining the current and previous iterations, and employ a parameterized iterative shrinkage-thresholding scheme to enhance performance and accelerate convergence. In the $k$-th iteration, we utilize \eqref{StepG} and \eqref{StepU} as extrapolated points, both of which are accelerated using \eqref{line}. This iterative process is instrumental in advancing the convergence and efficiency of the algorithm.

\begin{equation}
 \begin{aligned}
    & \tilde{\mathcal{G}}^{k} ={\mathcal{G}^{k}}+\omega_{k}\left(\mathcal{G}^{k}-\mathcal{G}^{k-1}\right),\\ & \omega_{k} = \min \big\{\frac{t^{k-1}-1}{t^{k}}, 0.999 \sqrt{\frac{L_\mathcal{G}^{k-1}}{L_\mathcal{G}^{k}}}\big\}, \ \text{for} \ k \geq 1 
    \end{aligned}\label{StepG}
\end{equation}
\begin{equation}
 \begin{aligned}
   & \tilde{\mathbf{U}}_{n}^{k} ={\mathbf{U}_{n}^{k}}+\omega_{k}\left({\mathbf{U}_{n}^{k}}-{\mathbf{U}}_{n}^{k-1}\right),\\ & \omega_{k} = \min \big\{\frac{t^{k-1}-1}{t^{k}}, 0.999 \sqrt{\frac{L_\mathbf{U}^{k-1}}{L_\mathbf{U}^{k}}}\big\} \ \text{for} \ k \geq 1  
    \end{aligned}\label{StepU}
\end{equation}
\begin{equation}
    t^{k} =\frac{0.8+\sqrt{4 (t^{k-1})^{2}+0.8}}{2}, \  t^{0} = 1. \label{line}
\end{equation}
Furthermore, we make sure that the value of $\Phi\left(\mathcal{G}^{k}, \{\mathbf{U}^{k}_n\}, \mathcal{X}^{k}\right)$ decreases before updating the extrapolated points $\tilde{\mathcal{G}}$, $\{\tilde{\mathbf{U}}_{n}\}$. If the condition is met, we update the Tucker components using equation \eqref{Update} under \eqref{StepG} - \eqref{line} hold.
\begin{equation}
    \begin{aligned}
    &\mathcal{G}^{k+1} =  \mathcal{S}_{\frac{\alpha}{L_{\mathcal{G}}^{k}}}\left(\tilde{\mathcal{G}}^{k}-\frac{1}{L_{\mathcal{G}}^{k}} \nabla_{\mathcal{G}} H\left(\tilde{\mathcal{G}}^{k}\right)\right), \\
    &\mathbf{U}_n^{k+1} = \mathcal{D}_{\frac{(1-\alpha)\omega_n}{L^k_{\mathbf{U}_n}}}\left(\tilde{\mathbf{U}}_n^k-\frac{1}{L^k_{\mathbf{U}_n}} \nabla_{\mathbf{U}_n} H\left(\tilde{\mathbf{U}}_n^k\right)\right)
    \end{aligned}\label{Update}.
\end{equation}

\begin{remark}
    Technically, we can consider allowing a larger step size by setting $L = 1.1 *L$ to accelerate the algorithm. However, balancing this acceleration with the algorithm's stability and convergence properties is crucial, as huge step sizes might lead to divergence or other convergence issues. Adjusting the step size is often part of the fine-tuning process in algorithm optimization.
\end{remark} 

If condition ~\eqref{Stop} is satisfied when the small value $\text{tol} = 1e^{-5}$, we calculate the complete tensor $\hat{\mathcal{X}} = \mathcal{T}_{\Omega} + {\mathcal{X}^{k+1}}_{\bar{\Omega}}$ as the TC completion result. 

\begin{equation}
    \left\|(\mathcal{X}^{k+1}-\mathcal{T})_{\Omega }\right\|_{F}/\left\|{\mathcal{T}}_{\Omega }\right\|_{F}<\text{tol}, \ \text{for some} \ k. \label{Stop}
\end{equation}

Based on the above algorithmic design process, we summarize the solution procedure in Algorithm~\ref{alg1} and analyze the PALM-based algorithm convergence in Theorem~\ref{Thm1}. 

\begin{algorithm}[!ht]
	\caption{PALM-based solver for the LRTuckerRep model}
	\label{alg1} 
	\begin{algorithmic}[1]
		\STATE \textbf{Input}: Incomplete tensor $\mathcal{T}$, observed entries $\Omega$.\\
		\STATE \textbf{Output}: Completion result $\hat{\mathcal{X}}$.
		\STATE Initialize $\mathcal{G}^0, \{\mathbf{U}^0_{n}\} $ ($ 1 \leq n \leq N $) randomly, $0 < \alpha < 1$, $\lambda = 1$, $K=500$;\\
            \STATE $\mathcal{X}^{0}_{\Omega}$ = ${\mathcal{T}}_{\Omega}$, $\mathcal{X}^{0}_{\bar{\Omega}}$ = mean( $\mathcal{T}_{\bar{\Omega}}$); \\
		\FOR{$k=0$ to $K$}
		\STATE $ \mathcal{G}^{k+1} =  \mathcal{S}_{\frac{\alpha}{L_{\mathcal{G}}^{k}}}\left(\tilde{\mathcal{G}}^{k}-\frac{1}{L_{\mathcal{G}}^{k}} \nabla_{\mathcal{G}} H\left(\tilde{\mathcal{G}}^{k}\right)\right) $, \\ $ \mathbf{U}_n^{k+1} = \mathcal{D}_{\frac{(1-\alpha)\omega_n}{L^k_{\mathbf{U}_n}}}\left(\tilde{\mathbf{U}}_n^k-\frac{1}{L^k_{\mathbf{U}_n}} \nabla_{\mathbf{U}_n} H\left(\tilde{\mathbf{U}}_n^k\right)\right) $;
		\STATE Update $\mathcal{X}^{k+1} = \mathcal{T}_{\Omega} + \left(\mathcal{G}^{k+1} \times_{n=1}^{N} \mathbf{U}_{n}^{k+1}\right)_{\bar{\Omega}}$; \\
		\IF {$\Phi\left(\mathcal{G}^{k+1}, \{\mathbf{U}^{k+1}_{n}\}, \mathcal{X}^{k+1}\right)$ is increasing}
		\STATE Re-update $\tilde{\mathcal{G}}^{k+1} = \mathcal{G}^{k}$ and $\tilde{\mathbf{U}}_n^{k+1} ={\mathbf{U}}_n^{k}$, respectively;
		\ELSE 
		\STATE Re-update $\tilde{\mathcal{G}}^{k+1}$ and $\tilde{\mathbf{U}}_n^{k+1}$ using \eqref{StepG} and \eqref{StepU} respectively;
		\ENDIF
            \STATE \textbf{until} $\left\|(\mathcal{X}^{k+1}-\mathcal{T})_{\Omega }\right\|_{F}/\left\|{\mathcal{T}}_{\Omega }\right\|_{F}<1e^{-4}$ is satisfied.
            \ENDFOR	
	\end{algorithmic}
\end{algorithm}

We establish the global convergence properties of the proposed PALM algorithm for solving the LRTuckerRep optimization problem.
\begin{thm} \label{Thm1}
Let \(\Theta^k = \{\mathcal{G}^k, \{\mathbf{U}^k_n\}\}\) be the sequence generated by Algorithm~\ref{alg1}, then we ensure that $\Theta^k$ globally converges to a critical point \(\hat{\Theta} = \{\hat{\mathcal{G}}, \{\hat{\mathbf{U}}_n\}\} \). 
\end{thm}

\begin{proof}
	\noindent \textbf{1) Square summable:} We express \eqref{PALM} as $\Phi(\Theta), \Theta = \{\mathcal{G}, \{\mathbf{U}_n\}\}$. The proximal linear updating rule indicates 
	\begin{equation*}
		\hat{\Theta} = \underset{\Theta}{\operatorname{argmin}} \ \left\langle\nabla_{\Theta} H(\tilde{\Theta}), \Theta -\tilde{\Theta} \right\rangle +\frac{L_{\Theta}}{2}\|\Theta -\tilde{\Theta}\|_{\mathrm{F}}^{2} + f(\Theta), \label{A1}
	\end{equation*}
	where $\tilde{\Theta}$ is the extrapolation point given by \eqref{StepG} and \eqref{line}. For any $\Theta^k = \{\mathcal{G}^k, \{\mathbf{U}^k_n\}\}$, it is worth noting that Algorithm~\ref{alg1} takes the bounded $L_{\Theta}^{k-1}$ as the Lipschitz constant of $\nabla_{\Theta} H({\Theta^k})$, then \eqref{A2} is satisfied, i.e., 
\begin{equation}
    \begin{aligned}
		H(\Theta^{k}) & \leq \ H(\Theta^{k-1}) + \left\langle\nabla_{\Theta} H(\Theta^{k-1}), \Theta^{k} -\Theta^{k-1}\right\rangle \\ & \quad + \frac{L^{k-1}_{\Theta}}{2}\|\Theta^{k} -\Theta^{k-1}\|_{\mathrm{F}}^{2}.
    \end{aligned}\label{A2}
\end{equation}
	Since the function $f$ is convex and the $H$ is strongly convex for other variables are fixed, \eqref{A2} ensures the proximal inequality \eqref{A3} holds \cite{PALM2014}. 
	\begin{equation}
		\Phi(\Theta) - \Phi(\hat{\Theta}) \geq \ \frac{L_{\Theta}}{2}\|\hat{\Theta} - \tilde{\Theta}\|_{\mathrm{F}}^{2} + L_{\Theta} \left\langle \tilde{\Theta} - \Theta, \hat{\Theta} - \tilde{\Theta} \right\rangle. \label{A3}
	\end{equation}
	Based on the results given by Proposition 1 and Proposition 2, $\nabla_{\Theta} H(\Theta)$ has bounded Lipschitz constant. Then for three successive $ \Theta^{k-2}, \Theta^{k-1}, $ $\Theta^{k}$ given by the updated step \eqref{StepG} and \eqref{StepU}, we have
	\begin{equation*}
		\begin{aligned}
			& \Phi(\Theta^{k-1}) - \Phi({\Theta}^k) \\ & \geq \ \frac{L_{\Theta}^{k-1}}{2}\|\Theta^k -\tilde{\Theta}^{k-1}\|_{\mathrm{F}}^{2} + L_{\Theta}^{k-1} \left\langle \tilde{\Theta}^{k-1} - \Theta^{k-1}, \Theta^k  - \tilde{\Theta}^{k-1} \right\rangle \\
			& \geq \ \frac{L_{\Theta}^{k-1}}{2} \|\Theta^{k-1} - \Theta^{k}\|_{\mathrm{F}}^{2} - \frac{L_{\Theta}^{k-2} \delta_{\omega}^2}{2} \|\Theta^{k-2} - \Theta^{k-1}\|_{\mathrm{F}}^{2}, \ \delta_{\omega} < 1.
		\end{aligned}
	\end{equation*}
	If \(\Phi\) increases (Algorithm~\ref{alg1} step 8), resetting \(\tilde{\Theta}^k = \Theta^{k-1}\) guarantees objective decrease. Summing the above inequality over $k$ from 1 to $K$, we have 
	\begin{equation*}
		\Phi(\Theta^{0}) - \Phi({\Theta}^K) \geq \sum_{k = 1}^{K}  \text{const.} \ \|\Theta^{k-1} - \Theta^{k}\|_{\mathrm{F}}^{2}.
	\end{equation*}
	Letting $K \to \infty $ and observing $\Phi$ is lower bounded (due to the regularization terms and data fidelity), then
	\begin{equation*}
		\sum_{k=1}^{\infty}\left\|\Theta^{k-1}-\Theta^{k}\right\|_{\mathrm{F}}^2 < \infty,
	\end{equation*}
	i.e., limit points of sequence $\{\Theta^{k}\}$ exist.
	
	\noindent \textbf{2) Subsequence convergence:} We set $\hat{\Theta}$ as a limit point of $\{\Theta^{k}\}$ depending on the square summable property. At the $k$th iteration of Algorithm~\ref{alg1}, we perform a re-update when $\Phi\left(\Theta_{k}\right) <  \Phi\left(\Theta_{k-1}\right)$, which assures the objective $\Phi$ non-increasing (we can also verify that using the results given by Lemma 3 in \cite{PALM2014}). To generate a stationary point $\hat{\Theta}$, we need to ensure the subgradient has a lower bound. For the given $\tilde{\mathcal{G}}, \{\tilde{\mathbf{U}}_n\}$, we have
	\begin{equation*}
		\hat{\mathcal{G}} = \ \underset{\mathcal{G}}{\operatorname{argmin}} \ \left\langle\nabla_{\mathcal{G}} H(\tilde{\mathcal{G}}, \{\hat{\mathbf{U}}_n\}), \mathcal{G} -\tilde{\mathcal{G}} \right\rangle  +\frac{\tilde{L}_{\mathcal{G}}}{2}\|\mathcal{G} -\tilde{\mathcal{G}}\|_{\mathrm{F}}^{2} + \alpha\|\mathcal{G}\|_{1}.
	\end{equation*}
	Letting $\tilde{\mathcal{G}} \to \hat{\mathcal{G}}$, we obtain $\tilde{L}_{\mathcal{G}} \to \hat{L}_{\mathcal{G}}$. Furthermore, the optimality condition holds for some $\mathbb{P}_{\mathcal{G}} \in \partial \|\mathcal{G}\|_{1}$
	\begin{equation}
		\nabla_{\mathcal{G}} H(\tilde{\mathcal{G}}) + \alpha \mathbb{P}_{\hat{\mathcal{G}}} = \hat{L}_{\mathcal{G}} \left(\tilde{\mathcal{G}} - \hat{\mathcal{G}}\right), \label{FOC1}
	\end{equation}
	Similarly, we have for all $\mathbf{U}_n$ that
	\begin{equation}
		\nabla_{\mathbf{U}_n} H(\tilde{\mathbf{U}}_n) + (1-\alpha) \omega_n \mathbb{P}_{\hat{\mathbf{U}}_n} = \hat{L}_{\mathbf{U}_n} \left(\tilde{\mathbf{U}}_n - \hat{\mathbf{U}}_n\right). \label{FOC2}
	\end{equation}
	Since $\nabla H$ is Lipschitz continuous and $\Theta^k$ is bounded, we have the subgradient lower bound holds for all $k=1, 2, \ldots, K$.
	\begin{equation*}
		\text{const.} \left\|\mathbb{P}^{k}\right\| \leq \left\|\Theta^{k}-\Theta^{k-1}\right\|, \quad \mathbb{P}^{k} \in \partial H\left(\Theta^{k}\right).
	\end{equation*}
	Referred by Lemma 5 in \cite{PALM2014}, the limit points set of the sequence $\Theta^{k}$ is compact. Then, the subsequence $\Theta^k$ converges to critical point $\hat{\Theta}$ \cite{PALM2014}. 
	
	\noindent \textbf{3) Global convergence via KL property:} The nuclear norm function $\{f_{n}\}$ is the piecewise analytic function, and $f_{N+1}$ is a semi-algebraic function since it is the finite sum of absolute functions. $H$ is a real polynomial function, hence semi-algebraic. Since analytic and semi-algebraic functions are KL functions \cite{KL2013}, we can ensure that the summation of KL functions $\Phi$ is the KL function. That is, $\Phi$ satisfies the KL property at $\hat{\Theta}$. The idea of our proof is to employ the KL property of $\Phi$ to show 
	\begin{equation}
		\|\Phi(\Theta)-\Phi(\hat{\Theta})\|^{\eta} \leq \mu \cdot \operatorname{dist}(\mathbf{0}, \partial \Phi(\Theta)), \text { for \ all } \Theta \in \mathcal{B}(\hat{\Theta}, \rho). \label{PALM-KL}
	\end{equation}
    where \(\rho, \mu > 0\), \(\eta \in [0,1)\) and
	\begin{equation*}
		\mathcal{B}(\hat{\Theta}, \rho) = \left\{\Theta:\|\Theta - \hat{\Theta}\|_{\mathrm{F}} \leq \rho\right\}.
	\end{equation*}
    Assume that $\Theta^k \in \mathcal{B}(\hat{\Theta}, \rho)$ for $0<k<K$, we verify that 
	\begin{equation*}
		\begin{aligned}
			\|\Theta^{K+1} - \hat{\Theta}\|_{\mathrm{F}} \leq & \ \|\Theta^{K} - \Theta^{K+1}\|_{\mathrm{F}} \\ & + \sum_{k=2}^{K-1} \left(\|\Theta^{K} - \Theta^{K+1}\|_{\mathrm{F}} + \|\Theta^{2} - \hat{\Theta}\|_{\mathrm{F}}\right) \\
			\leq & \ T \left(\Phi(\Theta^0)^{1-\theta} \right) + \|\Theta^{2} - \hat{\Theta}\|_{\mathrm{F}} < \rho,
		\end{aligned}
	\end{equation*}
	where $T$ is maximum value given by the bounded sequence $ \Theta^k $ \cite{SNTD2015}. Hence, $\Theta^{K+1} \in \mathcal{B}(\hat{\Theta}, \rho)$. Then, we ensure that \eqref{PALM-KL} holds using the induction method.
	
	Combining the subsequence convergence and \eqref{PALM-KL} holds, $\Theta^k$ converges globally to a critical point $\hat{\Theta}$ \cite{BCDXu2013, PALM2014} of the optimization problem \eqref{PALM}.
\end{proof}

\begin{remark}
    The boundedness of $\Theta^k$ ensures that both $L^k_\mathcal{G}$ and $L^k_{\mathbf{U}_n}$ are bounded (see Proposition 1 and Proposition 2). One way to make $\Theta^k$ bounded is to select $\nu>0$ and add $\mathbf{U}_{n} \leq \max \left(\nu,\|\mathcal{T}\|_{\infty}\right), \mathcal{G} \leq \max \left(\nu,\|\mathcal{T}\|_{\infty}\right)$ \cite{SNTD2015}.
\end{remark} 

\subsubsection{ProADM-based algorithm} In this section, we propose a parallel algorithm, named the proximal alternating direction multiplier (ProADM), to solve the Tucker constraint in the model \eqref{Model}. To address those multilinear constraints, we minimize the augmented Lagrange functions and incorporate a linearized proximal gradient update for each variable through the Gauss-Seidel iteration. Referred by \cite{nonconvexADMM2020, ProADMM2015}, we demonstrate that Algorithm~\ref{alg2} generates an approximate stationary point of the constrained problem \eqref{Model} given a sufficiently large penalty parameter.

We first utilize the augmented Lagrange function to obtain 
\begin{equation}
	\begin{aligned}
		& \underset{\mathcal{G},\{\mathbf{U}_{n}\},\mathcal{X},\mathcal{P}^{\mathcal{X}}, \mathcal{P}^{\Omega}, \mu}{\operatorname{min}} \mathbf{L}_{\mu}\left(\mathcal{G},\{\mathbf{U}_{n}\},\mathcal{X},\mathcal{P}^{\mathcal{X}}, \mathcal{P}^{\Omega}\right) =  \\ &  (1 - \alpha) \sum_{n=1}^{N} \omega_n \left\|\mathbf{U}_{n}\right\|_{*}+ \alpha \|\mathcal{G}\|_{1} + \sum_{n \in \Gamma} \frac{\beta_{n}}{2} \operatorname{tr}\left(\mathbf{U}_{n}^{\mathrm{T}} \mathbf{L}_n \mathbf{U}_{n}\right) \\
  & + \frac{\mu}{2} \left(\left\|\mathcal{X}-\mathcal{G} \times_{n=1}^{N} \mathbf{U}_{n}\right\|_{\mathrm{F}}^{2} + \left\|\mathcal{X}_{\Omega}-\mathcal{T}_{\Omega}\right\|_{\mathrm{F}}^{2}\right) \\
  & + \left\langle \mathcal{P}^{\mathcal{X}}, \mathcal{X}-\mathcal{G} \times_{n=1}^{N} \mathbf{U}_{n} \right\rangle + \left\langle \mathcal{P}^{\Omega}, \mathcal{X}_{\Omega}-\mathcal{T}_{\Omega} \right\rangle, 
	\end{aligned} \label{ProADM}
\end{equation}
where $\{\mu\}$ is a nondecreasing positive sequence, and $ \mathcal{P}^{\mathcal{X}}, \mathcal{P}^{\Omega} $ are the Lagrange multipliers. To find a closed-form solution, we adopt a proximal linear approximation solver to solve each subproblem to be near stationary. Specifically, we obtain the following iteration results for the current points $\mathcal{G}^{k},\{\mathbf{U}_{n}^{k}\},\mathcal{X}^{k},\mathcal{P}_{k}^{\mathcal{X}},\mathcal{P}_{k}^{\Omega}, \mu^{k}$.

\begin{itemize}
    \item Optimization of $\mathcal{G}$. With the other parameters fixed, $\mathcal{G}$ can be updated by solving $\mathbf{L}_{\mu^{k}}\left(\mathcal{G}\right)$, i.e., 
\begin{equation*}
	\begin{aligned}
		\underset{\mathcal{G}}{\operatorname{min}} & \ = \alpha \|\mathcal{G}\|_{1} + \frac{\mu^{k}}{2}\left\|\mathcal{X}^{k}+\frac{\mathcal{P}_{k}^{\mathcal{X}}}{\mu^{k}}-\mathcal{G} \times_{n=1}^{N} \mathbf{U}_{n}^{k}\right\|_{\mathrm{F}}^{2} \\ & \ = \alpha \|\mathcal{G}\|_{1} + f(\mathcal{G}). \label{G1}
	\end{aligned} 
\end{equation*}
Then, we use the quadratic approximation of $f$ to obtain the closed-form solution (the details are shown in Proposition 2).
\begin{equation}
	\mathcal{G}^{k+1} \approx \mathcal{S}_{\frac{\alpha}{ L_{\mathcal{G}}^{k}}}(\mathcal{G}^{k}-\frac{1}{L_{\mathcal{G}}^{k}} \nabla_{\mathcal{G}} f\left(\mathcal{G}^{k}\right)), \label{Gsol}
\end{equation}
where $\mathcal{G}^{k}$ is the previous iteration solution, and
\begin{equation*}
    L_{\mathcal{G}}^{k} = \mu^{k} \prod_{n=1}^{N}\left\|{{\mathbf{U}_{n}^{k}}}^\mathrm{T} \mathbf{U}_{n}^{k}\right\|_2, 
\end{equation*}
\begin{equation*}
\begin{aligned}
\nabla_{\mathcal{G}} f(\mathcal{G}^{k}) = & \ \mu^{k} \mathcal{G}^{k} \times_{n=1}^{N} {{\mathbf{U}_{n}^{k}}}^\mathrm{T}\mathbf{U}_{n}^{k} \\ & - \left(\mu^{k} \mathcal{X}^{k}+\mathcal{P}_{k}^{\mathcal{X}}\right) \times_{n=1}^{N} {{\mathbf{U}_{n}^{k}}}^\mathrm{T}.
 \end{aligned}
\end{equation*}

\item Optimization of $\mathbf{U}_n$. Guided by the results given by Proposition 3, we have the sub-Lagrange function concerning $\mathbf{U}_{p, p\neq n}^{k}, {\mathcal{G}}^{k+1}, {\mathcal{X}}^{k}$.
\begin{equation*}
	\underset{\mathbf{U}_{n}}{\operatorname{min}} \ = (1 - \alpha)\omega_n \left\|\mathbf{U}_{n}\right\|_{*} + f(\mathbf{U}_n).\label{U1}
\end{equation*}
Similarity, we use the local quadratic approximation
\begin{equation*}
\begin{aligned}
	f(\mathbf{U}_n) &= \frac{\beta_{n}}{2} \operatorname{tr}\left(\mathbf{U}_{n}^{\mathrm{T}} \mathbf{L}_n \mathbf{U}_{n}\right) \\ & \quad + \frac{\mu^{k}}{2} \operatorname{tr}\left(\mathbf{U}_{n} \mathbf{G}_{(n)}^{k+1} {\mathbf{V}_{n}^{\mathrm{T}}}^{k} \mathbf{V}_{n}^{k} {\mathbf{G}_{(n)}^{k+1}}^{\mathrm{T}} \mathbf{U}_{n}^{\mathrm{T}} \right) \\
	& \quad - \operatorname{tr}\left(\mathbf{U}_{n}  \left(\mu^{k}\mathbf{X}_{(n)}^{k}+\mathbf{P}_{k}^{\mathbf{X}_{(n)}}\right) \mathbf{V}_{n}^{k} {\mathbf{G}_{(n)}^{k+1}}^{\mathrm{T}} \right),
\end{aligned}
\end{equation*} 
to obtain approximation solution $\mathbf{U}_n^{k+1}$.
\begin{equation}
	\mathbf{U}_n^{k+1} \approx \ \mathcal{D}_{\frac{(1-\alpha)\omega_n}{L_{\mathbf{U}_n}^{k}}}(\mathbf{U}_n^{k}-\frac{1}{L_{\mathbf{U}_n}^{k}} \nabla_{\mathbf{U}_n} f\left({\mathbf{U}}_n^{k}\right)),  \label{Usol}
\end{equation}
with $L_{\mathbf{U}_{n}}^{k} = \left\|\mu^{k} \mathbf{G}_{(n)}^{k+1} {\mathbf{V}_{n}^{\mathrm{T}}}^{k} \mathbf{V}_{n}^{k} {\mathbf{G}_{(n)}^{k+1}}^{\mathrm{T}} \right\|_{2} + \beta_{n}\left\|{\mathbf{L}_n}\right\|_{2}$ and 
\begin{equation*}
\begin{aligned}
\nabla_{\mathbf{U}_{n}} f(\mathbf{U}_{n}) & = 
\mu^{k} {\mathbf{U}_{n}} \mathbf{G}_{(n)}^{k+1} {\mathbf{V}_{n}^{\mathrm{T}}}^{k} \mathbf{V}_{n}^{k} {\mathbf{G}_{(n)}^{k+1}}^{\mathrm{T}} + \beta_{n} \mathbf{L}_{n}{\mathbf{U}_{n}} \\ & \quad - \left(\mu^{k} \mathbf{X}_{(n)}^{k} + \mathbf{P}_{k}^{\mathbf{X}_{(n)}}\right) {\mathbf{V}}_n^{k}{\mathbf{G}_{(n)}^{k+1}}^{\mathrm{T}}.
\end{aligned}
\end{equation*}

\item Optimization of $\mathcal{X}$. We derive 
\begin{equation}
\begin{aligned}
&\mathcal{X}_{\Omega}^{k+1} = \frac{1}{2} \left(\mathcal{G}^{k+1} \times_{n=1}^{N} \mathbf{U}_{n}^{k+1} - \frac{\mathcal{P}_{k}^{\mathcal{X}}}{\mu_{k}}
+ \mathcal{T} - \frac{\mathcal{P}^{\Omega}_{k}}{\mu_{k}} \right)_{\Omega} , \\
&\mathcal{X}_{\bar{\Omega}}^{k+1} = \left(\mathcal{G}^{k+1} \times_{n=1}^{N} \mathbf{U}_{n}^{k+1} - \frac{\mathcal{P}_{k}^{\mathcal{X}}}{\mu_{k}} \right)_{{\bar{\Omega}}}. \label{Xsol}
\end{aligned}
\end{equation}

\item Updating the multipliers $\mathcal{P}^{\mathcal{X}}$ and $\mathcal{P}^{\Omega}$. 
\begin{equation}
\begin{aligned}
	& \mathcal{P}^{\mathcal{X}}_{k+1} = \ \mathcal{P}^{\mathcal{X}}_{k} + \mu^{k} \left(\mathcal{X}^{k+1} - \mathcal{G}^{k+1} \times_{n=1}^{N} \mathbf{U}_{n}^{k+1}\right), \\
       & \mathcal{P}^{\Omega}_{k+1} = \ \mathcal{P}^{\Omega}_{k} + \mu^{k} \left(\mathcal{X}_{\Omega}^{k+1} - \mathcal{T}_{\Omega}\right),\\
       & \mu^{k+1}= \rho \mu^{k}, \ \rho \in [1.1,1.2]. \label{M}
\end{aligned}
\end{equation}

\end{itemize}

The implementation process of the ProADM-based algorithm is summarized in Algorithm~\ref{alg2}. 
\begin{algorithm}[!ht]
	\caption{ProADM-based solver for the LRTuckerRep } \label{alg2} 
	\begin{algorithmic}[1]
		\STATE \textbf{Input}: Incomplete tensor $\mathcal{T}$, observed entries $\Omega$.\\
		\STATE \textbf{Output}: Completion result $\hat{\mathcal{X}}$.
		\STATE Initialize $\mathcal{G}^0, \{\mathbf{U}^0_{n}\} $ ($ 1 \leq n \leq N $) randomly, $0 < \alpha < 1$, $\mu = 1e^{-2}$, $K=500$;\\
		\STATE $\mathcal{X}^{0}_{\Omega}$ = ${\mathcal{T}}_{\Omega}$, $\mathcal{X}^{0}_{\bar{\Omega}}$ = mean( $\mathcal{T}_{\bar{\Omega}}$); \\
		\FOR{$k=1$ to $K$}
		\STATE Optimize $ \mathcal{G}^{k+1} $ via \eqref{Gsol} with other variables fixed;
		\STATE Optimize all $ \mathbf{U}_{n}^{k+1} $ via \eqref{Usol} with other variables fixed;
		\STATE Optimize $ \mathcal{X}^{k+1} $ via \eqref{Xsol} with other variables fixed; \\
		\STATE Update multipliers using \eqref{M}
		\STATE \textbf{until} \eqref{Stop} are satisfied.
		\ENDFOR
		\STATE \textbf{return} $ \hat{\mathcal{X}}_{\bar{\Omega}} = \mathcal{X}^{K}_{\bar{\Omega}} $, $\hat{\mathcal{X}}_{\Omega} = \mathcal{T}_{\Omega}$.
	\end{algorithmic} 
\end{algorithm}

\begin{thm} \label{Thm3}
	For sufficiently large $\mu$, the sequence \(\Theta^k = \{\mathcal{G}^k, \{\mathbf{U}^k_n\}, \mathcal{X}^k, \mathcal{P}_k^{\mathcal{X}}, \) \(\mathcal{P}_k^{\Omega}\}\) produced by Algorithm~\ref{alg2} globally converges to a critical point.
\end{thm}
The convergence analysis of Algorithm~\ref{alg2} is summarized in Theorem~\ref{Thm3}. We outline the proof strategy as follows: the augmented Lagrangian function is monotonically decreasing and the ProADM-based algorithm generates a bounded sequence \(\Theta^k = \{\mathcal{G}^k, \{\mathbf{U}^k_n\}, \mathcal{X}^k, \mathcal{P}_k^{\mathcal{X}}, \mathcal{P}_k^{\Omega}\}\); every limit point \(\hat{\Theta} = \{\hat{\mathcal{G}}, \{\hat{\mathbf{U}}_n\},\) \( \hat{\mathcal{X}}, \hat{\mathcal{P}}^{\mathcal{X}}, \hat{\mathcal{P}}^{\Omega} \} \) satisfies the necessary first-order optimality conditions; the KL property to demonstrate that \(\Theta^k\) converges globally to a critical point. The detailed proof can be found in \cite{nonconvexADMM2020}.

\section{Experiments results} 
\label{sec: experiments}
\subsection{Experimental setup}
In this section, we compare the performance of our proposal and baselines using multi-dimensional data under random missing (RM) scenarios. For the image data, we consider the mean signal-to-noise ratio (MPSNR) and mean structural similarity (MSSIM) of all spatial images to assess the image data inpainting performance \cite{LS2023}. 
\begin{equation}
	\operatorname{MPSNR} = 10 \cdot \log_{10} \frac{\left(\mathcal{X}_{\max }\right)^{2}}{\left\|\hat{\mathcal{X}}-\mathcal{X}_{\text {true }}\right\|_{\mathrm{F}}^{2} /\left|\bar{\Omega}\right|}, \label{PSNR}
\end{equation}
where $\hat{\mathcal{X}}$ is the estimated tensor and $\mathcal{X}_{\text{true}}$ is the ground-truth tensor.
\begin{equation}
	\operatorname{MSSIM}(x, y) = \frac{{(2\mu_x \mu_y + C_1)(2\sigma_{xy} + C_2)}}{{(\mu_x^2 + \mu_y^2 + C_1)(\sigma_x^2 + \sigma_y^2 + C_2)}}, \label{SSIM}
\end{equation}
where $\mathcal{X}_{\max }$ denotes the maximum value in the $\mathcal{X}_{\text{true}}$. A more significant value indicates better performance. 

For multi-dimensional traffic data, the evaluation metrics include the mean absolute percentage error (MAPE) and normalized mean absolute error (NMAE). 
\begin{equation}
	\begin{aligned}
		\mathrm{MAPE} =\frac{1}{n} \sum_{i=1}^{n}\left|\frac{y_{i}-\hat{y}_{i}}{y_{i}}\right| \times 100, \quad
		\mathrm{NMAE} = \frac{\sum_{i=1}^{n} \left|y_{i}-\hat{y}_{i}\right|}{\sum_{i=1}^{n} \left|{y_{i}}\right|}
	\end{aligned} \label{M-N}
\end{equation}
where ${y_{i}}$ and ${\hat{y}_{i}}$ are actual values and imputed values, respectively.  The lower values indicate better results.

All experiments were performed using MATLAB 2023a on a workstation equipped with a Windows 10 64-bit operating system, an Intel(R) Xeon(R) W-2123 CPU with 3.60 GHz and 64 GB RAM. Our MATLAB codes are available on request.

\subsection{Synthetic data experiment} % 存在的必要性
The synthetic tensor was generated by multi-dimensional Gaussian functions, and its incomplete version was created by randomly removing values. We investigate the performance of the baselines with a sample ratio (SR) of 10\%.
\subsubsection{Parameter analysis}
To evaluate the model's sensitivity, we set $\alpha = [0.01, 0.05, 0.1, 0.3, 0.5, 0.7, 0.9] $ and $\lambda = [0.1, 0.5, 1, 10, 1e^2, 1e^3]$. Fig.~\ref{fig1} (left) presents heat maps of MPSNR values achieved by the ELT model on synthetic tensor data. The results indicate that the ELT model performs exceptionally well, mainly when using the parameters $\alpha=0.01$ and $\lambda=1$. In particular, the parameter $\alpha$ has a more significant impact on the final results, and we empirically set the penalty parameter $\lambda=1$ for all experiments. The results presented in Fig.~\ref{fig1} (right) demonstrate that the initial strategy effectively reduces errors compared to other randomly initialized approaches. 

\begin{figure}[htbp]
	\centering
	\includegraphics[width=1\linewidth]{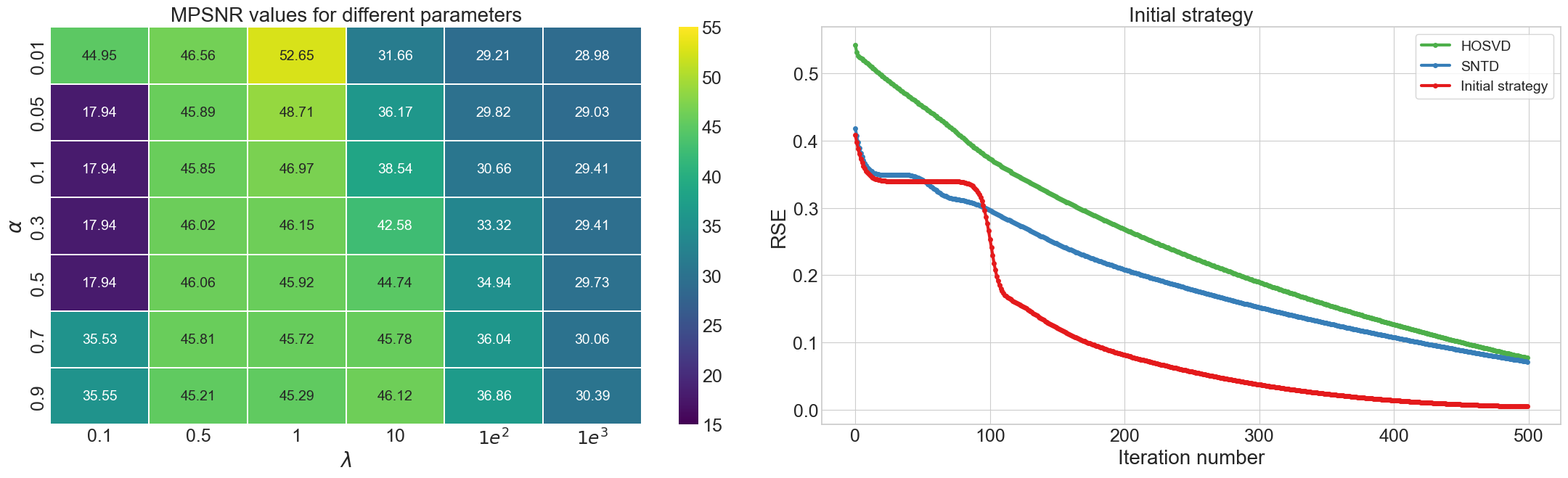}
	\caption{Parameter analysis for PALM algorithm}          \label{fig1}	
\end{figure}
\subsubsection{Convergence}
The PALM- and ProADM-based algorithms have theoretical convergence results, as presented in Theorem~\ref{Thm1} and Theorem~\ref{Thm3}. Here, we analyze the numerical convergence using the relative change error and the RSE value. Fig.~\ref{fig2} illustrates the relative error curve for the recovered tensor over two successive iterations of our proposed model about the number of iteration steps. The results indicate that the relative error values achieved by \eqref{Model} gradually converge to zero, signifying the numerical convergence of Algorithm~\ref{alg1} and Algorithm~\ref{alg2}. Furthermore, decreasing the RSE curves versus the iteration number shows that the two algorithms are numerically converged.
\begin{figure}[htbp]
	\centering
	\includegraphics[width=1\linewidth]{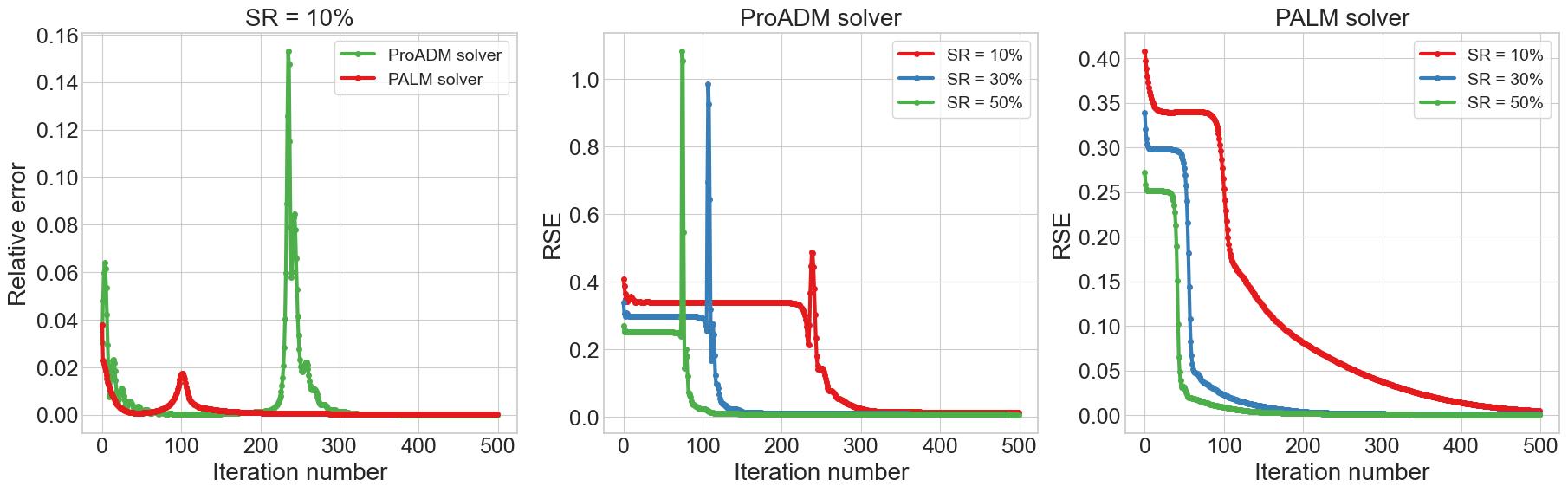} 
	\caption{Variation of relative error and root square error (RSE) values with the number of iteration steps for both PALM- and ProADM-based algorithms }	
	\label{fig2}
\end{figure}

\subsubsection{Complexity}
Suppose that ${\mathcal{X}}$ and ${\mathcal{G}}$ $\in \mathbb{R}^{I_{1} \times \ldots \times I_{N}}$, we have the basic computational complexity: the computational cost of $ {\mathbf{U}_{n}^{\mathrm{T}}\mathbf{U}_{n}} $ is $\mathcal{O}(I_{n}^3)$ and the mode-n product with the matrix ${\mathbf{U}_{n}}$ of tensor ${\mathcal{G}}$ is ${\mathcal{O}(\sum_{n=1}^N \prod_{i=1}^{n} I_{i} \prod_{j=1}^{N} I_{j} )}$. Furthermore, we reformulate the Kronecker product in ${\mathbf{G}^{n}_{\mathbf{V}}} = {\mathbf{G}_{(n)}} \mathbf{V}_{n}^{\mathrm{T}}$, and its computational cost is $\mathcal{O}\left(\sum_{n=1}^{N}\left(\prod_{i=1}^{n} I_{i}\right)\left(\prod_{j=n}^{N} I_{j}\right)\right)$. Considering the proposed PALM-based optimization for Tucker core updating, the computation of $\nabla_\mathcal{G} H\left(\mathcal{G}\right)$ requires
\begin{equation*}
	\mathcal{O}\left(\sum_{n=1}^{N} I_{n}^{3} +\sum_{n=1}^{N} I_{n} \prod_{i=1}^{N} I_{i}+\sum_{n=1}^{N}\left(\prod_{i=1}^{n} I_{i}\right)\left(\prod_{j=n}^{N} I_{j}\right)\right).
\end{equation*}
and the computational complexity of $\nabla_{\mathbf{U}_{n}} H(\mathbf{U}_{n})$ requires
\begin{equation*}
    \mathcal{O}\left(I_{n}\left(\prod_{i=1}^{n} I_{i}\right) + I_{n}^3\right) + \mathcal{O}\left(\prod_{i=1}^{n} I_{i}\right) + \mathcal{O}\left(I^3_{n}\right) + \mathcal{O}\left(\mathbf{G}^{n}_{\mathbf{V}}\right). 
\end{equation*}
So, the per-iteration time complexity of the Algorithm~\ref{alg1} as
\begin{equation}
\begin{aligned}
		\mathcal{O} \left( \left(N + 1\right) \sum_{n=1}^{N}\left(\prod_{i=1}^{n} I_{i}\right)\left(\prod_{j=n}^{N} I_{j}\right)\right). \label{Computation}
	\end{aligned}
\end{equation}

\subsection{Multi-dimensional data completion}
\subsubsection{Image data inpainting}
This section uses the LRTuckerRep model to recover third-order image data within RM scenarios. We evaluated results using metrics such as MPSNR and MSSIM, where higher scores denote superior performance. The experimental baselines included SNNTV \cite{SNNTV2017}, leveraging the nuclear norm sum; STMac \cite{STMac2016}, based on the factorization of tensor unfolding matrices; SPC \cite{SPC2016}, using CP decomposition; SparsityTD \cite{ARTD2023}, based on tensor sparsity; tCTV \cite{LS2023}, using tubal rank; ESP \cite{ESP2020} and LRSETD \cite{LRSETD2024}, both employing Tucker decomposition. The hyperparameters of the baselines were manually and optimally configured.

The RGB image `Stockton' is selected from the USC-SIPI datasets \footnote{\url{https://sipi.usc.edu/database/database.php}} with dimensions of $512 \times 512 \times 3$. For the RGB inpainting task, we validate that incorporating a smooth structural prior significantly improves performance. Fig.~\ref{fig3} (a) shows that our factor gradient depicts the image's local similarity, and the PALM-based algorithm exhibits superior accuracy. Fig.~\ref{fig3} (b) illustrates that LRTuckerRep achieves higher MPSNR and MSSIM values compared to other Tucker-based baselines. Furthermore, we assess the recovery results of LRTuckerRep compared to other baselines that incorporate both low-rank and smooth priors. The results in Fig.~\ref{fig3} (c) demonstrate that our model outperforms others, particularly in scenarios with high missing rates. Moreover, Tab.~\ref{IMAGESOTA} demonstrates that our model outperforms others, particularly in scenarios with high missing rates. Fig.~\ref{fig4} visually shows that our proposal can still work in extreme missing scenarios, such as the missing cases 95\%. It is reasonable to conclude that the LRTuckerRep method is efficient based on its strong recovery performance.
\begin{figure}[htbp]
	\centering
	\includegraphics[width=1\linewidth]{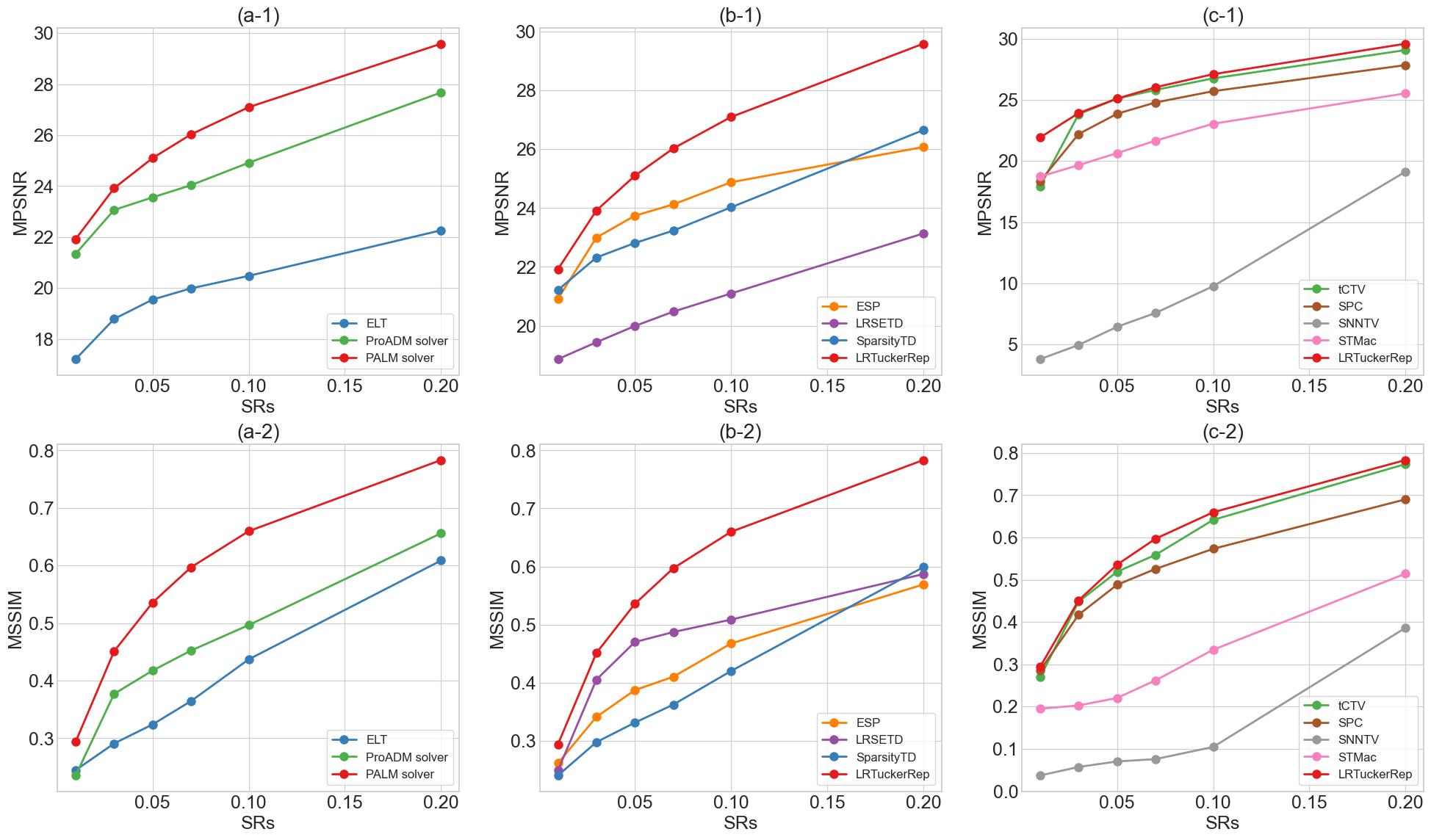}
	\caption{Comparison results of MPSNR and MSSIM values under different SRs for image restoration. (a) Smoothness enhances model performance, and the PALM-based algorithm performs better. (b) Our proposal achieves a higher MPSNR value than other Tucker-based TC models. (c) The proposed LRTuckerRep model outperforms other baselines when SRs are lower than 20\% }	
	\label{fig3}
\end{figure}

We also use MRI data \footnote{\url{http://brainweb.bic.mni.mcgill.ca/brainweb/}} with dimensions of $181 \times 217 \times 40$ for testing purposes. Given the inherent smoothness of the MRI images in all three spatial modes \cite{SPC2016}, we validate that enhanced low-rank Tucker model and factor gradients can improve the performance of tensor completion models. Fig.~\ref{fig5} shows the performance of different low-rank tensor completion methods, with visual results showing improved recovery performance with the incorporation of smoothness, and LRTuckerRep yields superior results. Furthermore, results in Fig.~\ref{fig5} (c) showcase our model outperforming others, particularly in scenarios with high missing rates. From Tab.~\ref{IMAGESOTA} and Fig.~\ref{fig4}, it is evident that LRTuckerRep successfully recovers missing MRI data, even in 95\% missing ratio, significantly outperforming other methods in terms of MPSNR and MSSIM.
\begin{figure}[htbp]
	\centering
	\includegraphics[width=1\linewidth]{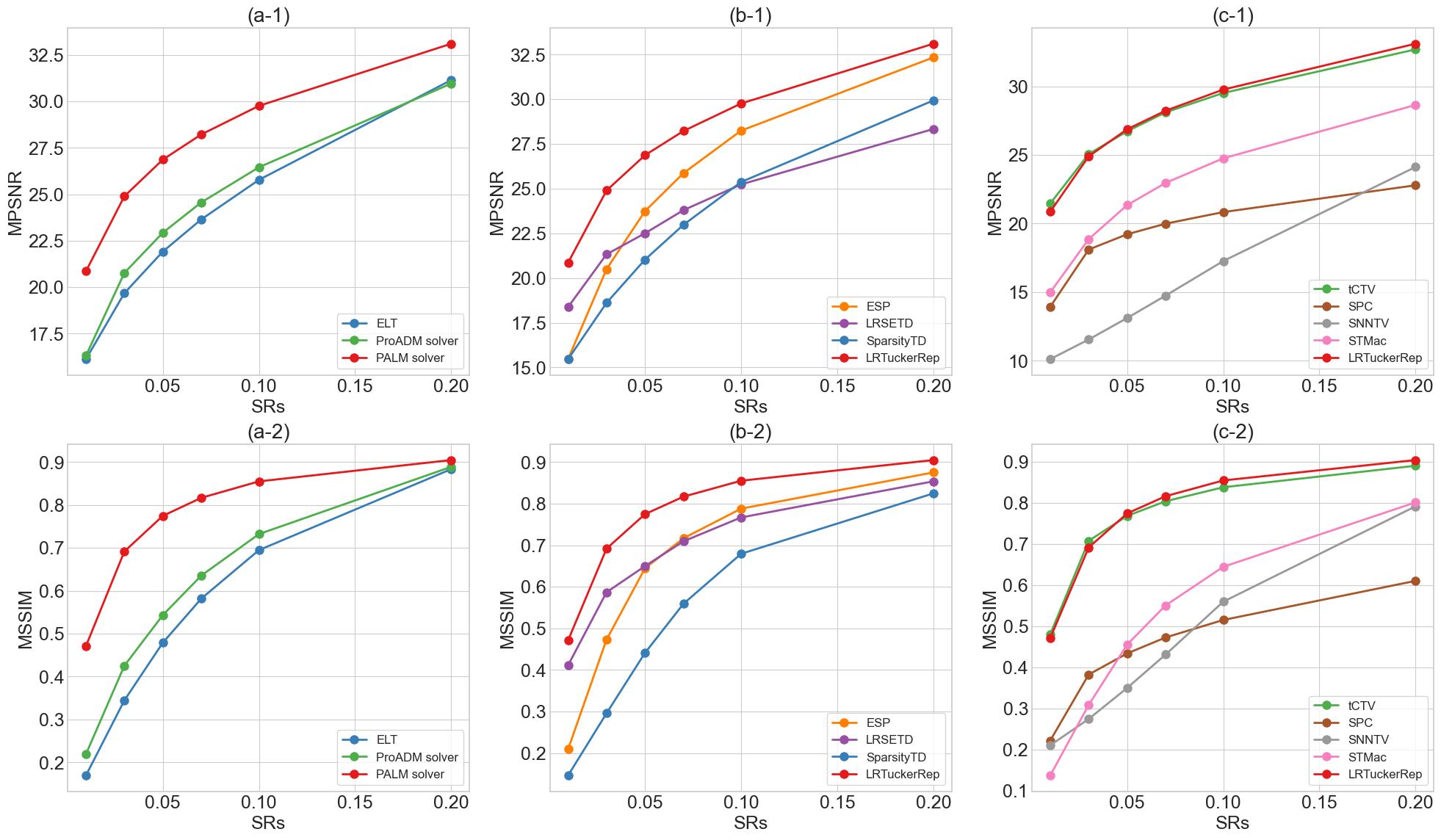}
	\caption{Comparison results of MPSNR and MSSIM values under different SRs for the MRI recovery }	
	\label{fig5}
\end{figure}

\begin{table}[!htbp]
	\centering
	\caption{MPSNR values comparison of all baseline methods on image data inpainting. The best results are highlighted in bold fonts }
	\label{IMAGESOTA}
	\begin{tabular}{lcccccc}
		\toprule 
		\multirow{2}{*}{\textbf{Methods}} & \multicolumn{2}{c}{ RGB } & \multicolumn{2}{c}{ MRI } & \multicolumn{2}{c}{ MSI }  \\
		%\cline {2 - 4}\cline {5 - 7}\cline {8 - 10} \cline {11 - 13} \cline {14 - 16}
		\cmidrule(lr){2-3} \cmidrule(lr){4-5} \cmidrule(lr){6-7}
		& 1\% & 5\% & 1\% & 5\% & 1\% & 5\%  \\
		\midrule 
		\textbf{LRTuckerRep} & \textbf{21.91} & \textbf{26.99} & 20.86 & \textbf{26.97} & \textbf{22.55} & 25.05 \\
            tCTV & 17.93 & 24.24 & \textbf{21.46} & 26.88 & 22.01 & \textbf{26.50}  \\
		SparsityTD & 21.65 & 24.15 & 18.89 & 24.93 & 21.29 & 24.78 \\
		LRSETD & 21.54 & 24.25 & 18.44& 22.76 & 21.09 & 23.45 \\
		ESP & 10.57 & 10.88 & 12.06 & 12.31 & 17.87 & 18.05 \\
		SPC & 20.10 & 24.31 & 14.61 & 21.88 & 19.99 & 23.68 \\
		SNNTV & 3.82 & 6.55 & 10.12 & 13.13 & 15.42 & 18.69  \\
		STMac & 18.72 & 20.64 & 15.05 & 21.85 & 16.84 & 20.54 \\
		\bottomrule
	\end{tabular}
\end{table}

\begin{figure*}[!htbp]
	\centering
	\includegraphics[width=1\linewidth]{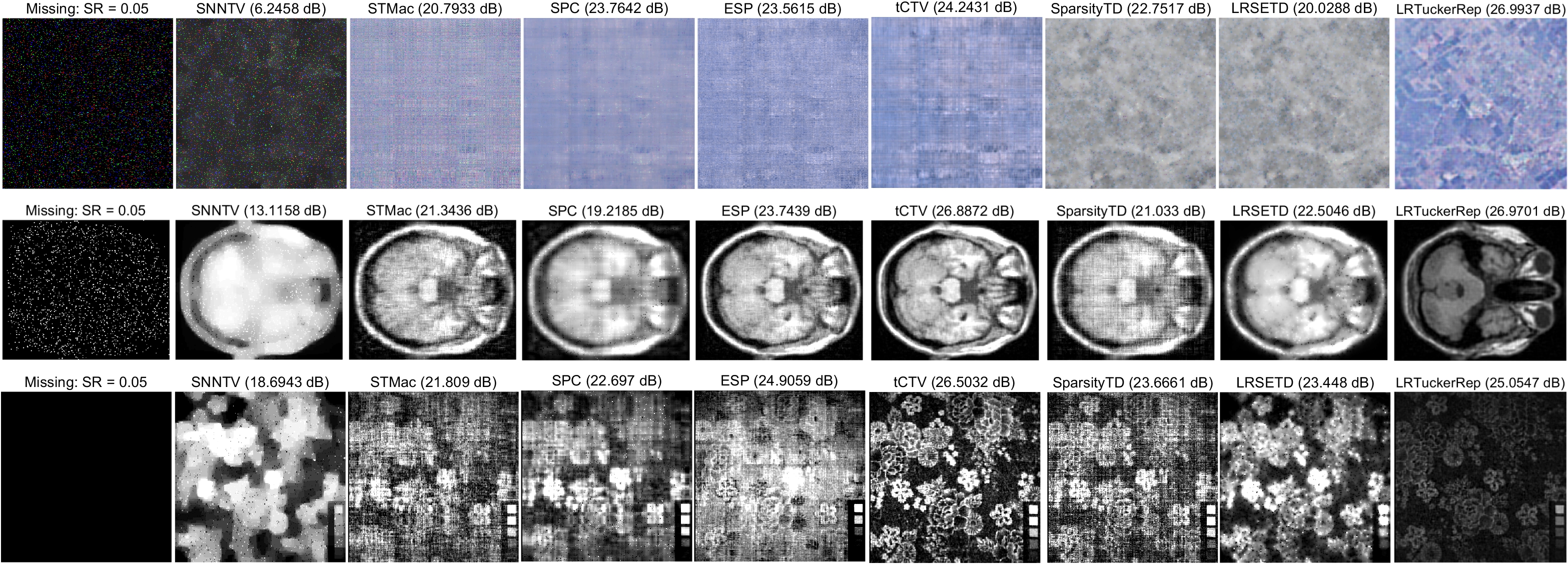}
	\caption{Visualization of image inpainting results under 95\% missing scenario }
    \label{fig4}
\end{figure*}

We further evaluate the performance of different methods using the multispectral image `Cloth' \footnote{\url{https://www.cs.columbia.edu/CAVE/databases/multispectral/}} ($256\times256\times31$). Fig.~\ref{fig6} demonstrates that incorporating smoothness improves model performance, and the MPSNR and MSSIM values of the inpainting results obtained by our proposed method surpass those of other baselines (except for the tCTV method). Fig.~\ref{fig4} also showcases a band of the test images reconstructed by the LRTuckerRep model and other baselines under extreme missing scenarios. Our proposed model produces visually superior results compared to the compared methods, with the inpainting results closest to the ground truths.
\begin{figure}[htbp]
	\centering
	\includegraphics[width=1\linewidth]{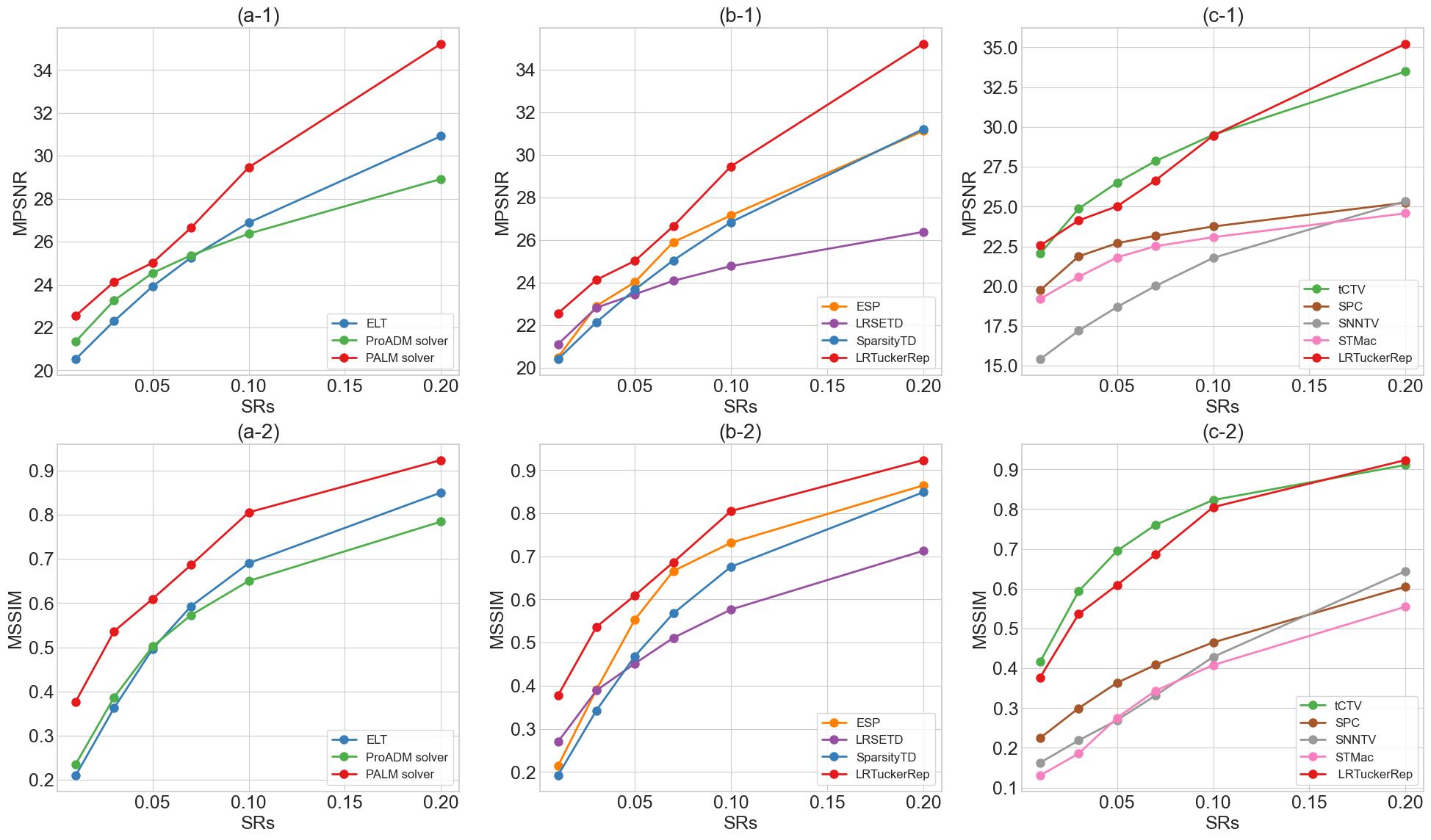}
	\caption{Comparison results of MPSNR and MSSIM values under different SRs for the MSI `Cloth' inpainting }	
	\label{fig6}
\end{figure}

\subsubsection{Traffic data imputation}
This section presents the experimental results on two real traffic datasets \footnote{\url{https://doi.org/10.5281/zenodo.7725126}}, namely Guangzhou urban traffic speed (\textbf{UTD}: $214\times144\times7$) and Abilene Internet traffic flow (\textbf{NTD}: $121\times 288\times 7$). To evaluate the imputation performance, we use two common metrics: MAPE and NMAE, where lower values indicate better performance. For comparison, we select seven baseline methods: the Tucker-based method \cite{LRSETD2024}, the sparsity-based method \cite{STRTD2023}, the tubal-based method \cite{TubalLSTC2021}, the TT-based method \cite{stTT2021}, the SNN-based method \cite{LATC2021}, the CP-based method \cite{SPC2016}, and the matrix-based method \cite{STMac2016}. These methods are chosen to demonstrate the robustness and efficiency of the proposed LRTuckerRep model.

To assess the efficacy of factor gradients for traffic data imputation, we explore their impact within the LRTuckerRep model. The results shown in Fig.\ref{fig7} illustrate the imputation performance on the UTD and NTD datasets. Notably, we observe that utilizing factor gradients along the first two tensor modes helps exploit the inherent spatiotemporal properties of traffic data, reducing imputation errors for both datasets. Compared to NTD data, the LRTuckerRep model performs better in capturing spatiotemporal traffic patterns in the UTD dataset. This indicates that traffic data collected from roadway sensors have more interpretable spatiotemporal features, making the factor gradients more effective in capturing spatial and temporal patterns. Furthermore, we evaluate the performance of our proposed algorithms, as outlined in Algorithm \ref{alg1} and Algorithm~\ref{alg2}. As shown in Fig.~\ref{fig7} (a-1, b-1), the PALM-based algorithm exhibits superior accuracy but requires more computational time.
\begin{figure}[!htbp]
	\centering
	\includegraphics[width=1\linewidth]{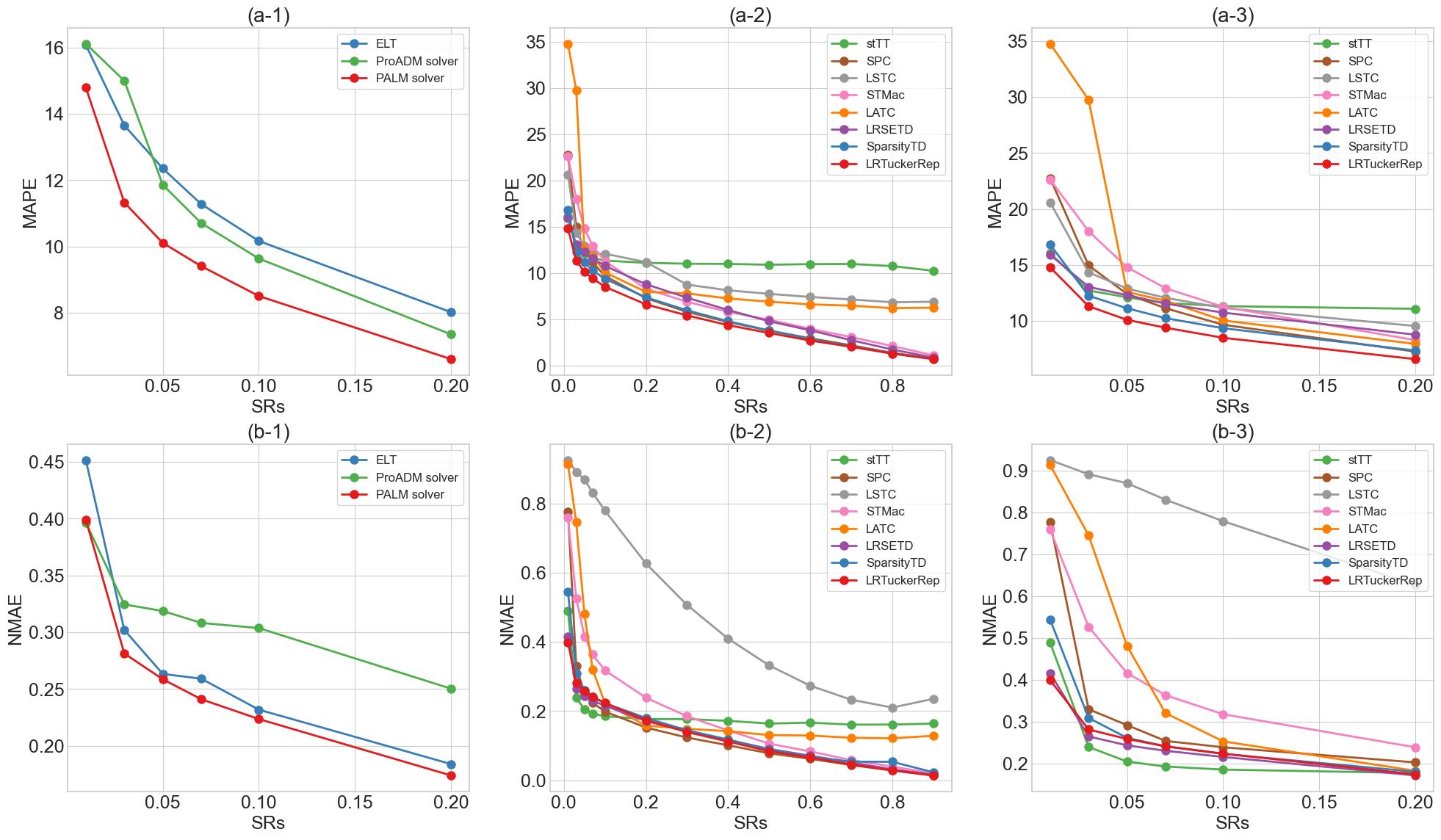}
	\caption{MAPE and NMAE values for different SRs under RM scenarios for datasets UTD (a) and NTD (b) }	
	\label{fig7}
\end{figure}

Fig.\ref{fig7} (a-2, b-2) presents the MAPE and NMAE values for the imputation of UTD and NTD using the seven baseline methods, evaluated under different sample ratios. The results demonstrate LRTuckerRep's competitive performance. Notably, LRTuckerRep stands out, achieving comparable accuracy even when the sample ratio is below 20\% (see Fig.\ref{fig7} (a-3, b-3)). The results highlight the robustness of the LRTuckerRep method, which continues to perform well even when a large proportion of the data is missing. In conclusion, LRTuckerRep is highly effective for imputing traffic data. It demonstrates its strong generalization ability and outperforms other baseline methods, even in extreme scenarios where 95\% of the data observations are missing. Through these experiments, we demonstrate that the proposed LRTuckerRep method significantly improves traffic data imputation tasks, making it an essential tool for applications requiring accurate and reliable imputation techniques in the presence of missing data.

\section{Conclusion}
\label{sec: conclusions}
This paper presents a Tucker-based prior modeling framework in tensor completion and proposes a novel low-rank Tucker Representation (LRTuckerRep) model. The proposed LRTuckerRep integrates low rankness and smoothness priors within a nonconvex least squares formulation for \( L_2 \) error minimization. The key innovation lies in a novel low-rank representation via Tucker core sparsity, combined with nuclear norm regularization on factor matrices, which avoids costly SVD operations on unfolding matrices. Simultaneously, a parameter-free Laplacian regularization is employed to adaptively capture smoothness within the factor spaces, eliminating the need for manual parameter tuning. Two globally convergent algorithms are developed to solve the LRTuckerRep model efficiently. Experimental results demonstrate that LRTuckerRep achieves high accuracy and strong robustness, even under extreme missing rates (e.g., 95\% missing rate).  

Future directions include reducing the computational cost of large-scale matrix operations through fast Fourier transforms \cite{Tatsuya2022} or tensor T-product decompositions \cite{tSVD2022}. Additionally, the proposed LRTuckerRep model represents a promising tool for multi-dimensional data processing tasks, such as tensor robust principal component analysis \cite{LS2023} and pattern discovery \cite{JASA2021}.

\section*{Acknowledgments}
This research was supported by the SUSTech Presidential Postdoctoral Fellowship and the China Postdoctoral Science Foundation under grant 2025M773057.

\section*{Appendix}
\subsection*{Proof of Proposition 1} 
\begin{proof}
For the vectorization form, $H\left(\mathcal{G}\right)$ is equivalent to
\begin{equation*}
    H\left(\mathcal{G}\right) = \frac{\lambda}{2} \left\|\left(\otimes_{n=N}^{1} \mathbf{U}_{n}^{\mathrm{T}} \mathbf{U}_{n}\right) \operatorname{vec}(\mathcal{G}) - \operatorname{vec}(\mathcal{X})\right\|_{\mathrm{F}}^{2}.
\end{equation*}
By utilizing the properties of the Kronecker product, we have the Hessian matrix $ \operatorname{vec}\left(\nabla^2_{\mathcal{G}} H(\mathcal{G})\right) = \lambda \otimes_{n=N}^{1} \mathbf{U}_{n}^{\mathrm{T}} \mathbf{U}_{n}$, which is positive semidefinite and ensures $H(\mathcal{G})$ is convex. We utilize the affine approximation of $H(\mathcal{G})$ and add the proximal term using a Lipschitz constant $L_{\mathcal{G}}$ step size at the extrapolated point $\tilde{\mathcal{G}}$, then its local quadratic approximation 
\begin{equation}
    \begin{aligned} 
    \hat{\mathcal{G}} \approx  & \ \underset{\mathcal{G}}{\operatorname{argmin}}\left\langle\nabla_{\mathcal{G}} H(\tilde{\mathcal{G}}), \mathcal{G}-\tilde{\mathcal{G}}\right\rangle+\frac{L_\mathcal{G}}{2}\|\mathcal{G}-\tilde{\mathcal{G}}\|_{F}^{2}+\alpha\|\mathcal{G}\|_{1} \\
    = & \ \underset{\mathcal{G}}{\operatorname{argmin}} \ \frac{L_{\mathcal{G}}}{2}\left\|\mathcal{G}-(\tilde{\mathcal{G}}-\frac{1}{L_{\mathcal{G}}} \nabla_{\mathcal{G}} H\left(\tilde{\mathcal{G}}\right))\right\|_{\mathrm{F}}^{2}  + \alpha \|\mathcal{G}\|_{1} \\
    = & \ \operatorname{prox}_{L_{\mathcal{G}}/\alpha}^{\|\cdot\|_1}(\tilde{\mathcal{G}}-\frac{1}{L_{\mathcal{G}}} \nabla_{\mathcal{G}} H\left(\tilde{\mathcal{G}}\right)) \\ 
    = & \ \mathcal{S}_{\frac{\alpha}{L_{\mathcal{G}}}}(\tilde{\mathcal{G}}-\frac{1}{L_{\mathcal{G}}} \nabla_{\mathcal{G}}H\left(\tilde{\mathcal{G}}\right)),
    \end{aligned}
\end{equation}
which finds a closed-form solution. For any given $\mathcal{G}_1$ and $\mathcal{G}_2$, we have 
\begin{equation*}
\begin{aligned}
 & \left\|\operatorname{vec}\left(\nabla_{\mathcal{G}} f_{N+1}(\mathcal{G}_1)\right) - \operatorname{vec}\left(\nabla_{\mathcal{G}} f_{N+1}(\mathcal{G}_2)\right)\right\|_{\mathrm{F}}^2 \\
	\leq & \ \lambda \left\| \otimes_{n=N}^{1} \mathbf{U}_{n}^{\mathrm{T}} \mathbf{U}_{n}\right\|_2 \left\|\operatorname{vec}(\mathcal{G}_1) - \operatorname{vec}(\mathcal{G}_2)\right\|_{\mathrm{F}}^2 \\
  = & \lambda \prod_{n=1}^{N}\left\|\mathbf{U}_{n}^{\mathrm{T}} \mathbf{U}_{n}\right\|_2
 \left\|\operatorname{vec}(\mathcal{G}_1) - \operatorname{vec}(\mathcal{G}_2)\right\|_{\mathrm{F}}^2.
\end{aligned} 
\end{equation*}   
On any bounded set $\{\mathbf{U}_{n}\}$, the Lipschitz constant is bounded and can be expressed as $L_\mathcal{G} =\lambda \prod_{n=1}^{N}\left\|\mathbf{U}_{n}^{\mathrm{T}} \mathbf{U}_{n}\right\|_2$.
\end{proof}	

\subsection*{Proof of Proposition 2} 
\begin{proof}
    We have the gradient of $ H\left(\mathbf{U}_n\right) $
\begin{equation*}
\nabla_{\mathbf{U}_{n}} H(\mathbf{U}_{n}) = \lambda \left({\mathbf{U}_{n}} {\mathbf{G}^{n}_{\mathbf{V}}} {\mathbf{G}^{n}_{\mathbf{V}}}^{\mathrm{T}} - \mathbf{X}_{(n)} {\mathbf{G}^{n}_{\mathbf{V}}}^{\mathrm{T}}\right) + \beta_{n} \mathbf{L}_n{\mathbf{U}_{n}},
\end{equation*}
where ${\mathbf{G}^{n}_{\mathbf{V}}} = {\mathbf{G}_{(n)}} \mathbf{V}_{n}^{\mathrm{T}}$. On the one hand, the Hessian matrix of $H(\mathbf{U}_{n})$ is given by $\nabla_{\mathbf{U}_{n}}^2 H(\mathbf{U}_{n}) $ $= \lambda
\mathbf{G}^{n}_{\mathbf{V}} {\mathbf{G}^{n}_{\mathbf{V}}}^{\mathrm{T}} +  \beta_{n} \mathbf{L}_{n}$. Since the functions $\mathbf{G}^{n}_{\mathbf{V}} {\mathbf{G}^{n}_{\mathbf{V}}}^{\mathrm{T}}$ and $\mathbf{L}_{n}$ are both positive and semi-definite, it shows that $H(\mathbf{U}_{n})$ is convex. The local quadratic approximation around the extrapolated point $\tilde{\mathbf{U}}_n$ ensures that
\begin{equation}
\begin{aligned}
    \hat{\mathbf{U}}_n & \approx \ \underset{{\mathbf{U}_n}}{\operatorname{argmin}} \ \left\langle\nabla_{\mathbf{U}_n} H(\tilde{\mathbf{U}}_n), \mathbf{U}_n-\tilde{\mathbf{U}}_n\right\rangle \\ & \quad +\frac{L_{\mathbf{U}_n}}{2}\|\mathbf{U}_n-\tilde{\mathbf{U}}_n\|_{F}^{2} + (1 - \alpha) \ \omega_n \left\|\mathbf{U}_{n}\right\|_{*} \\
    = & \ \underset{\mathbf{U}_{n}}{\operatorname{argmin}} \ \frac{L_{\mathbf{U}_{n}}}{2}\left\|\mathbf{U}_{n}-(\tilde{\mathbf{U}}_{n}-\frac{1}{L_{\mathbf{U}_{n}}} \nabla_{\mathbf{U}_{n}} H\left(\tilde{\mathbf{U}}_{n}\right))\right\|_{\mathrm{F}}^{2} \\ & \ + (1 - \alpha) \ \omega_n \left\|\mathbf{U}_{n}\right\|_{*}\\
    = & \ \operatorname{prox}_{\frac{L_{\mathbf{U}_n}}{(1-\alpha)\omega_n}}^{\|\cdot\|_*} (\tilde{\mathbf{U}}_n-\frac{1}{L_{\mathbf{U}_n}} \nabla_{\mathbf{U}_n} H\left(\tilde{\mathbf{U}}_n\right)) \\
    = & \ \mathcal{D}_{\frac{(1-\alpha)\omega_n}{L_{\mathbf{U}_n}}}(\tilde{\mathbf{U}}_n-\frac{1}{L_{\mathbf{U}_n}} \nabla_{\mathbf{U}_n} H\left(\tilde{\mathbf{U}}_n\right)),
\end{aligned}
\end{equation}
which gives a proximal operator. On the other hand, the Lipschitz constant of $\nabla_{\mathbf{U}_{n}} H$ can be calculated separately. For any two matrices ${\mathbf{U}_{n}^1}, {\mathbf{U}_{n}^2}$, we have 
\begin{equation*}
	\begin{array}{l}
		\left\|\nabla_{\mathbf{U}_{n}} H(\mathbf{U}_{n}^1) - \nabla_{\mathbf{U}_{n}} H(\mathbf{U}_{n}^2) \right\|_{\mathrm{F}}^{2} \\
		= \left\|\lambda\left({\mathbf{U}_{n}^1} - {\mathbf{U}_{n}^2}\right) {\mathbf{G}^{n}_{\mathbf{V}}} {\mathbf{G}^{n}_{\mathbf{V}}}^{\mathrm{T}} - \beta_{n} \mathbf{L}_{n} \left({\mathbf{U}_{n}^1} - {\mathbf{U}_{n}^2}\right) \right\|_{\mathrm{F}}^{2} \\
		\leq \left\|\lambda\left({\mathbf{U}_{n}^1} - {\mathbf{U}_{n}^2}\right) {\mathbf{G}^{n}_{\mathbf{V}}} {\mathbf{G}^{n}_{\mathbf{V}}}^{\mathrm{T}} \right\|_{\mathrm{F}}^{2} + \left\|\beta_{n} \mathbf{L}_{n} \left({\mathbf{U}_{n}^1} - {\mathbf{U}_{n}^2}\right) \right\|_{\mathrm{F}}^{2}. \\
	\end{array} 
\end{equation*}
More specifically,
\begin{equation*}
	\begin{array}{l}
	\left\|\lambda\left({\mathbf{U}_{n}^1} - {\mathbf{U}_{n}^2}\right) {\mathbf{G}^{n}_{\mathbf{V}}} {\mathbf{G}^{n}_{\mathbf{V}}}^{\mathrm{T}} \right\|_{\mathrm{F}}^{2} \\
		= \operatorname{tr}\left( \lambda{\mathbf{G}^{n}_{\mathbf{V}}} {\mathbf{G}^{n}_{\mathbf{V}}}^{\mathrm{T}} \left({\mathbf{U}_{n}^1} - {\mathbf{U}_{n}^2}\right)^{\mathrm{T}}\left({\mathbf{U}_{n}^1} - {\mathbf{U}_{n}^2}\right){\mathbf{G}^{n}_{\mathbf{V}}} {\mathbf{G}^{n}_{\mathbf{V}}}^{\mathrm{T}}\right) \\
		\leq \lambda \left\|\mathbf{G}^{n}_{\mathbf{V}} {\mathbf{G}^{n}_{\mathbf{V}}}^{\mathrm{T}}\right\|_{2}\left\|{\mathbf{U}_{n}^1} - {\mathbf{U}_{n}^2} \right\|_{\mathrm{F}}^{2},
	\end{array} 
\end{equation*}
and
\begin{equation*}
	\begin{array}{l}
        \left\|\beta_{n} \mathbf{L}_{n} \left({\mathbf{U}_{n}^1} - {\mathbf{U}_{n}^2}\right) \right\|_{\mathrm{F}}^{2} \\
		= \operatorname{tr}\left(\beta_{n} {\mathbf{L}_{n}}^{\mathrm{T}} \left({\mathbf{U}_{n}^1} - {\mathbf{U}_{n}^2}\right)^{\mathrm{T}}\left({\mathbf{U}_{n}^1} - {\mathbf{U}_{n}^2}\right)\beta_{n}{\mathbf{L}_{n}}\right) \\
		\leq \beta_{n} \left\|{\mathbf{L}_{n}}\right\|_{2} \left\|{\mathbf{U}_{n}^1} - {\mathbf{U}_{n}^2} \right\|_{\mathrm{F}}^{2}.
	\end{array} 
\end{equation*}
Since ${\left\|\mathbf{G}_{\mathbf{V}}\right\|_{2}}$ and ${\left\|\mathbf{L}_{n}\right\|_{2}}$ denote spectral norms, ${\nabla_{\mathbf{U}_{n}} H(\mathbf{U}_{n})}$ is Lipschitz continuous, and the Lipschitz constant $L_{\mathbf{U}_{n}}$ is bounded when $\{\mathbf{U}_{p}, p \neq n\}, \mathcal{G}$ is bounded.
\end{proof}

\bibliographystyle{elsarticle-num}
\bibliography{references.bib}

\begin{thebibliography}{10}
\expandafter\ifx\csname url\endcsname\relax
  \def\url#1{\texttt{#1}}\fi
\expandafter\ifx\csname urlprefix\endcsname\relax\def\urlprefix{URL }\fi
\expandafter\ifx\csname href\endcsname\relax
  \def\href#1#2{#2} \def\path#1{#1}\fi

\bibitem{IP2019}
Q.~Song, H.~Ge, J.~Caverlee, X.~Hu, Tensor completion algorithms in big data
  analytics, ACM Trans. Knowl. Discov. Data 13~(1), doi:10.1145/3278607 (2019).

\bibitem{SP2017}
N.~D. Sidiropoulos, L.~De~Lathauwer, X.~Fu, K.~Huang, E.~E. Papalexakis,
  C.~Faloutsos, Tensor decomposition for signal processing and machine
  learning, IEEE Trans. Signal Process. 65~(13) (2017) 3551--3582,
  doi:10.1109/TSP.2017.2690524.

\bibitem{TC2009}
J.~Liu, P.~Musialski, P.~Wonka, J.~Ye, Tensor completion for estimating missing
  values in visual data, in: 2009 IEEE 12th International Conference on
  Computer Vision (ICCV), 2009, pp. 2114--2121.

\bibitem{GLNP2022}
X.~Zhao, J.~Yang, T.~Ma, T.~Jiang, M.~K. Ng, T.~Huang, Tensor completion via
  complementary {G}lobal, {L}ocal, and {N}onlocal priors, IEEE Trans. Image
  Process. 31 (2022) 984--999, doi:10.1109/TIP.2021.3138325.

\bibitem{LS2023}
H.~Wang, J.~Peng, W.~Qin, J.~Wang, D.~Meng, Guaranteed tensor recovery fused
  low-rankness and smoothness, IEEE Trans. Pattern Anal. Mach. Intell. 45~(9)
  (2023) 10990--11007, doi:10.1109/TPAMI.2023.3259640.

\bibitem{HaLRTC2013}
J.~Liu, P.~Musialski, P.~Wonka, J.~Ye, Tensor completion for estimating missing
  values in visual data, IEEE Trans. Pattern Anal. Mach. Intell. 35~(1) (2013)
  208--220, doi:10.1109/TPAMI.2012.39.

\bibitem{Tmac2015}
Y.~Xu, R.~Hao, W.~Yin, Z.~Su, Parallel matrix factorization for low-rank tensor
  completion, Inverse Probl Imaging 9~(2) (2015) 601--624,
  doi:10.3934/ipi.2015.9.601.

\bibitem{TT2017}
J.~A. Bengua, H.~N. Phien, H.~D. Tuan, M.~N. Do, Efficient tensor completion
  for color image and video recovery: {L}ow-rank tensor train, IEEE Trans.
  Image Process. 26~(5) (2017) 2466--2479, doi:10.1109/TIP.2017.2672439.

\bibitem{BayesianCP2022}
T.~Hiromu, Q.~B. Zhao, H.~Hidekata, T.~Yokota, Bayesian tensor completion and
  decomposition with automatic {CP} rank determination using {MGP} shrinkage
  prior, SN Computer Science 3~(225) (2022) 1--17,
  doi:10.1007/s42979-022-01119-8.

\bibitem{tSVDTR2022}
B.~Li, X.~Zhao, T.~Ji, X.~Zhang, T.~Huang, Nonlinear transform induced tensor
  nuclear norm for tensor completion, J Sci Comput 92~(3) (2022) 83,
  doi:10.1007/s10915-022-01937-1.

\bibitem{LogTucker2017}
T.~Ji, T.~Huang, X.~Zhao, T.~Ma, L.~Deng, A non-convex tensor rank
  approximation for tensor completion, Appl. Math. Model. 48 (2017) 410--422,
  doi.org/10.1016/j.apm.2017.04.002.

\bibitem{KBR2018}
Q.~Xie, Q.~Zhao, D.~Meng, Z.~Xu, Kronecker-{B}asis-{R}epresentation based
  tensor sparsity and its applications to tensor recovery, IEEE Trans. Pattern
  Anal. Mach. Intell. 40~(8) (2018) 1888--1902, doi:10.1109/TPAMI.2017.2734888.

\bibitem{LATC2021}
X.~Chen, M.~Lei, N.~Saunier, L.~Sun, Low-rank autoregressive tensor completion
  for spatiotemporal traffic data imputation, IEEE Trans Intell Transp Syst
  23~(8) (2021) 1--10, doi:10.1109/TITS.2021.3113608.

\bibitem{Auxiliary2012}
A.~Narita, K.~Hayashi, R.~Tomioka, H.~Kashima, Tensor factorization using
  auxiliary information, Data Min Knowl Discov 25 (2012) 298–324,
  doi:10.1007/s10618-012-0280-z.

\bibitem{ARTD2023}
W.~Gong, Z.~Huang, L.~Yang, Accurate regularized {T}ucker decomposition for
  image restoration, Appl. Math. Model. 123~(11) (2023) 75--86,
  doi:10.1016/j.apm.2023.06.031.

\bibitem{SNNTV2017}
X.~Li, Y.~Ye, X.~Xu, Low-rank tensor completion with total variation for visual
  data inpainting, in: Proceedings of the AAAI Conference on Artificial
  Intelligence (AAAI), Vol.~31, 2017, pp. 2210--2216,
  doi:10.1609/aaai.v31i1.10776.

\bibitem{STMac2016}
T.~Ji, T.~Huang, X.~Zhao, T.~Ma, G.~Liu, Tensor completion using total
  variation and low-rank matrix factorization, Inf. Sci. 326 (2016) 243--257,
  doi:10.1016/j.ins.2015.07.049.

\bibitem{SPC2016}
T.~Yokota, Q.~Zhao, A.~Cichocki, Smooth {PARAFAC} decomposition for tensor
  completion, IEEE Trans. Signal Process. 64~(20) (2016) 5423--5436,
  doi:10.1109/TSP.2016.2586759.

\bibitem{TTTV2019}
M.~Ding, T.~Huang, T.~Ji, X.~Zhao, J.~Yang, Low-rank tensor completion using
  matrix factorization based on tensor train rank and total variation, J Sci
  Comput 81~(2) (2019) 941--964, doi:10.1007/s10915-019-01044-8.

\bibitem{TubalLSTC2021}
X.~Chen, Y.~Chen, N.~Saunier, L.~Sun, Scalable low-rank tensor learning for
  spatiotemporal traffic data imputation, Transp Res Part C Emerg Technol 129
  (2021) 103226, doi:10.1016/j.trc.2021.103226.

\bibitem{ESP2020}
J.~Xue, Y.~Zhao, W.~Liao, J.~C.-W. Chan, S.~G. Kong, Enhanced sparsity prior
  model for low-rank tensor completion, IEEE Trans. Neural Networks Learn.
  Syst. 31~(11) (2020) 4567--4581, doi:10.1109/TNNLS.2019.2956153.

\bibitem{SBCD2022}
Q.~Yu, X.~Zhang, Y.~Chen, L.~Qi, Low {T}ucker rank tensor completion using a
  symmetric block coordinate descent method, Numer Linear Algebra Appl 30~(3)
  (2023) e2464, doi:10.1002/nla.2464.

\bibitem{BGCP2015}
Q.~Zhao, L.~Zhang, A.~Cichocki, Bayesian {CP} factorization of incomplete
  tensors with automatic rank determination, IEEE Trans. Pattern Anal. Mach.
  Intell. 37~(9) (2015) 1751–1763, doi:10.1109/TPAMI.2015.2392756.

\bibitem{tSVD2017}
Z.~Zhang, S.~Aeron, Exact tensor completion using t-{SVD}, IEEE Trans. Signal
  Process. 65~(6) (2017) 1511--1526, doi:10.1109/TSP.2016.2639466.

\bibitem{PAM2015}
M.~Razaviyayn, M.~Hong, Z.~Luo, A unified convergence analysis of block
  successive minimization methods for nonsmooth optimization, SIAM J. Optim.
  23~(2) (2013) 1126--1153, doi:10.1137/120891009.

\bibitem{ProADMM2015}
G.~Li, T.~Pong, Global convergence of splitting methods for nonconvex composite
  optimization, SIAM J. Optim. 25~(4) (2015) 2434--2460, doi:10.1137/140998135.

\bibitem{BCDXu2013}
Y.~Xu, W.~Yin, A block coordinate descent method for regularized multiconvex
  optimization with applications to nonnegative tensor factorization and
  completion, SIAM J. Imaging Sci. 6~(3) (2013) 1758--1789,
  doi:10.1137/120887795.

\bibitem{PALM2014}
J.~Bolte, S.~Sabach, M.~Teboulle, Proximal alternating linearized minimization
  for nonconvex and nonsmooth problems, Math. Program. 146~(7) (2014)
  459–494, doi:10.1007/s10107-013-0701-9.

\bibitem{LRSETD2024}
C.~Pan, C.~Ling, H.~He, L.~Qi, Y.~Xu, A low-rank and sparse enhanced {T}ucker
  decomposition approach for tensor completion, Appl. Math. Comput. 465 (2024)
  128432, doi:10.1016/j.amc.2023.128432.

\bibitem{BayesianTucker2022}
J.~Xue, Y.~Zhao, Y.~Bu, J.~C.-W. Chan, S.~G. Kong, When {L}aplacian scale
  mixture meets three-layer transform: {A} parametric tensor sparsity for
  tensor completion, IEEE Trans Cybern 52~(12) (2022) 13887--13901,
  doi:10.1109/TCYB.2021.3140148.

\bibitem{gHOI2016}
Y.~Liu, F.~Shang, W.~Fan, J.~Cheng, H.~Cheng, Generalized higher order
  orthogonal iteration for tensor learning and decomposition, IEEE Trans.
  Neural Networks Learn. Syst. 27~(12) (2016) 2551--2563,
  doi:10.1109/TNNLS.2015.2496858.

\bibitem{STDC2014}
Y.~Chen, C.-T. Hsu, H.-Y.~M. Liao, Simultaneous tensor decomposition and
  completion using factor priors, IEEE Trans. Pattern Anal. Mach. Intell.
  36~(3) (2014) 577--591, doi:10.1109/TPAMI.2013.164.

\bibitem{SVT2013}
J.~Cai, E.~J. Cand\`{e}s, Z.~Shen, A singular value thresholding algorithm for
  matrix completion, SIAM J. Optim. 20~(4) (2010) 1956--1982,
  doi:10.1137/080738970.

\bibitem{SNTD2015}
Y.~Xu, Alternating proximal gradient method for sparse nonnegative {T}ucker
  decomposition, Math Program Comput 5~(3) (2015) 455–500,
  doi:10.1007/s12532-014-0074-y.

\bibitem{KL2013}
H.~Attouch, J.~Bolte, B.~Svaiter, Convergence of descent methods for
  semi-algebraic and tame problems: proximal algorithms, forward-backward
  splitting, and regularized {G}auss-{S}eidel methods, Math. Program. 146~(137)
  (2013) 91--129, doi:10.1007/s10107-011-0484-9.

\bibitem{STRTD2023}
W.~Gong, Z.~Huang, L.~Yang, Spatiotemporal regularized {T}ucker decomposition
  approach for traffic data imputation, arXivDoi:arXiv:2305.06563 (2023).

\bibitem{APG2015}
H.~Li, Z.~Lin, Accelerated proximal gradient methods for nonconvex programming,
  in: Proceedings of the 28th International Conference on Neural Information
  Processing Systems (NIPS), 2015, p. 379–387.

\bibitem{FISTA2022}
J.~Liang, T.~Luo, C.-B. Sch\"{o}nlieb, Improving “{F}ast {I}terative
  {S}hrinkage-{T}hresholding {A}lgorithm”: {F}aster, {S}marter, and
  {G}reedier, SIAM J Sci Comput 44~(3) (2022) A1069--A1091,
  doi:10.1137/21M1395685.

\bibitem{nonconvexADMM2020}
W.~Gao, D.~Goldfarb, F.~E. Curtis, {ADMM} for multiaffine constrained
  optimization, Optimization Methods and Software 35~(2) (2020) 257--303,
  doi:10.1080/10556788.2019.1683553.

\bibitem{stTT2021}
Z.~Zhang, Y.~Chen, H.~He, L.~Qi, A tensor train approach for internet traffic
  data completion, Ann Oper Res 06 (2021) 12--19,
  doi:10.1007/s10479-021-04147-4.

\bibitem{Tatsuya2022}
R.~Yamamoto, H.~Hontani, A.~Imakura, T.~Yokota, Fast algorithm for low-rank
  tensor completion in delay-embedded space, in: 2022 IEEE/CVF Conference on
  Computer Vision and Pattern Recognition (CVPR), 2022, pp. 2048--2056.

\bibitem{tSVD2022}
H.~He, C.~Ling, W.~Xie, Tensor completion via a generalized transformed tensor
  t-product decomposition without t-{SVD}, J Sci Comput 93~(2) (2022) 1--35,
  doi:10.1007/s10915-022-02006-3.

\bibitem{JASA2021}
R.~Chen, Y.~Dan, C.~Zhang, Factor models for high-dimensional tensor time
  series, J Am Stat Assoc 117~(537) (2022) 94--116,
  doi:10.1080/01621459.2021.1912757.

\end{thebibliography}



\end{document}